\pdfoutput=1
\documentclass{article} % For LaTeX2e
\usepackage[utf8]{inputenc} % allow utf-8 input
\usepackage[T1]{fontenc}    % use 8-bit T1 fonts
\usepackage{hyperref}       % hyperlinks
\usepackage{url}            % simple URL typesetting
\usepackage{booktabs}       % professional-quality tables
\usepackage{amsfonts}       % blackboard math symbols
\usepackage{nicefrac}       % compact symbols for 1/2, etc.
\usepackage{microtype}      % microtypography
\usepackage{geometry}
\usepackage{authblk}
\usepackage{latexsym}

\usepackage{enumerate}
\usepackage[inline]{enumitem}
\usepackage{amsmath,amssymb}
\usepackage{amsfonts,dsfont}
\usepackage{nicefrac}
\usepackage{microtype}
\usepackage{mathtools}

\usepackage{caption}
\usepackage{subcaption}
\usepackage{wrapfig}
\usepackage{enumitem}
\usepackage{algorithm}
\usepackage[noend]{algorithmic}
\usepackage[normalem]{ulem}
\usepackage{amssymb}
\usepackage{multicol}
\usepackage{adjustbox}
\usepackage{multirow}
\usepackage{color}
\usepackage{xspace}
\usepackage{CJKutf8}

\PassOptionsToPackage{numbers}{natbib}
\usepackage{natbib}

\usepackage{mkolar_definitions}
\usepackage{grinchmacros}

\newcommand{\perch}{\textsc{Perch}\xspace}
\newcommand{\alg}{\textsc{Grinch}\xspace}

\newcommand{\algrotate}{\textsc{Rotate}\xspace}
\newcommand{\records}{data points\xspace}
\newcommand{\record}{data point\xspace}

\newcommand{\cstar}{\ensuremath{\mathcal{C}^\star}}

\newcommand{\hac}{\textsc{HAC}\xspace}

\newcommand{\greedy}{\textsc{Online}\xspace}
\newcommand{\exact}{\textsc{HAC}\xspace}
\newcommand{\mbhac}{\textsc{MB-HAC}\xspace}
\newcommand{\graft}{\texttt{graft}\xspace}
\newcommand{\grafting}{\texttt{graft}ing\xspace}
\newcommand{\grafts}{\texttt{graft}s\xspace}
\newcommand{\rotate}{\texttt{rotate}\xspace}
\newcommand{\rst}{\texttt{restruct}\xspace}

\newcommand{\hof}{linkage function\xspace}
\newcommand{\hofs}{linkage functions\xspace}

\newcommand{\lvs}[1]{\ensuremath{\texttt{lvs}(#1)}}

\newcommand{\ancs}[1]{\ensuremath{\texttt{ancs}(#1)}}

\newcommand{\conNN}{\ensuremath{\texttt{constrNN}}\xspace}
\newcommand{\lca}[2]{\ensuremath{\texttt{lca}(#1, #2)}}
\newcommand{\merge}[2]{\ensuremath{\texttt{merge}(#1, #2)}}
\newcommand{\pur}[1]{\ensuremath{\texttt{pur}}(#1)\xspace}
\newcommand{\sib}[1]{\ensuremath{\texttt{s}}(#1)\xspace}
\newcommand{\parent}[1]{\ensuremath{\texttt{par}}(#1)\xspace}
\newcommand{\aunt}[1]{\ensuremath{\texttt{aunt}}(#1)\xspace}

\newcommand{\condition}[2]{\ensuremath{f(#1, #2) > \max[f(#1, \sib{#1}), f(#2, \sib{#2})]}}

\def\BibTeX{{\rm B\kern-.05em{\sc i\kern-.025em b}\kern-.08emT\kern-.1667em\lower.7ex\hbox{E}\kern-.125emX}}

\begin{document}

\title{Scalable Hierarchical Clustering with Tree Grafting}
\author{Nicholas Monath\thanks{The first two authors contributed
    equally.} $^1$, Ari Kobren$^{*1}$, Akshay Krishnamurthy$^2$, \\Michael
  Glass$^3$, Andrew McCallum$^1$}
\affil{\{nmonath,akobren,akshay,mccallum\}@cs.umass.edu, mrglass@us.ibm.com}
\affil{College of Information and Computer Sciences \\ University of
  Massachusetts Amherst$^1$}
\affil{Microsoft Research, New York City$^2$}
\affil{IBM, New York City$^3$}

\maketitle

\begin{abstract}
  We introduce \alg, a new algorithm for large-scale, non-greedy
  hierarchical clustering with general linkage functions that compute arbitrary similarity between two point
  sets.  The key components of \alg are its \rotate and \graft subroutines
  that efficiently reconfigure the hierarchy as new points arrive,
  supporting discovery of clusters with complex structure.  \alg
  is motivated by a new notion of separability for clustering with
  \hofs: we prove that when the model is consistent with a
  ground-truth clustering, \alg is guaranteed to produce a cluster
  tree containing the ground-truth, independent of data arrival order.
  Our empirical results on benchmark and author coreference datasets
  (with standard and learned \hofs) show that \alg is more accurate
  than other scalable methods, and orders of magnitude faster than
  hierarchical agglomerative clustering.
\end{abstract}
\section{Introduction}
\label{sec:intro}
Best-first, bottom-up, hierarchical agglomerative clustering (\hac) is
one of the most widely-used clustering algorithms, proving effective
for a wide variety of applications such as analyzing gene expression
data~\cite{eisen1998cluster}, community detection in social
networks~\cite{blundell2013bayesian}, and scientific author
disambiguation~\cite{culotta2007author}. One capability that
contributes significantly to \hac's prevalence is that it can be used
to construct a clustering according to any cluster-level scoring
function, also known as a \emph{\hof}~\cite{kohli2009robust,
  lee2012joint}. This is crucial for applications such as entity
resolution, in which the quality of a cluster is typically a learned
function of a group of \records~\cite{culotta2007author,
  singh2011large, wick2012discriminative}.

While effective, \hac requires $O(n^2\log n)$ computation for general
\hofs, making it infeasible to run on datasets of even moderate
size. One option for circumventing this computational problem is to
use an online or mini-batch variant of the algorithm.  However, both
\hac variants make irrecoverable, greedy merges and are thus sensitive
to data arrival order.  Non-greedy, \emph{incremental algorithms}
provide a more robust alternative to their online
counterparts~\cite{kobren2017hierarchical, zhang1996birch}. Like
online approaches, incremental methods consume \records, one at a time,
but when new data arrives, incremental algorithms can also revisit
previous clustering decisions. However, current incremental algorithms
fail in two ways: they are only capable of reconsidering clustering
decisions at a \emph{local} level and they do not support arbitrary
\hofs \cite{kobren2017hierarchical, zhang1996birch}.

In this paper we introduce \alg, a hierarchical, incremental
(non-greedy) clustering algorithm that can cluster with any \hof. \alg
builds a cluster tree over the incoming \records, one at a time,
attempting to keep similar \records near one another in the
tree. Robustness to suboptimal data arrival order is achieved by
employing both local \emph{and global} tree rearrangements.  Local
rearrangements are performed using a \rotate subroutine, which
recursively swaps a child with its aunt. Global rearrangements are
performed via a \graft subroutine, in which \alg may steal a subtree
from one part of the hierarchy and merge it with another similar, but
distant, subtree. Grafting is a key for both our theoretical and
empirical results, and supports the discovery of clusters that exhibit
(single or sparse) linked structures---an important feature of
clustering algorithms used in practice~\cite{ester1996density}.

Theoretically, we define a notion of \emph{model-based separation}
that characterizes the relationship between a \hof and a dataset.  For
generality, we adopt a graph-theoretic formalism, where \records
correspond to vertices of an unknown graph whose connected components
form a ground truth clustering. Model-based separation suggests that
the \hof value is high for two item sets if the induced subgraph is
connected (see Subsection~\ref{subsec:hsep}).  We prove that under
this condition, the ground-truth clusters are a tree-consistent
partition of the hierarchy built by \alg.

In experiments, we show that \alg is efficient and builds trees with
higher dendrogram purity than other clustering algorithms on large
scale datasets.  The experiments are performed with a common and
important \hof---average linkage---as well as a \hof that
measures the cosine similarity between
two cluster centroid representations. We also perform experiments on two author
coreference datasets using learned \hofs, and demonstrate that \alg is
more efficient and accurate than the baselines. Our experiments
reveal that \alg dominates competitors that only make local tree
rearrangements, highlighting the power of the \graft subroutine and
the robustness of \alg.
\section{Linkage Functions for Clustering}
\label{sec:clustering}
\emph{Clustering} is the problem of constructing a partition $\Ccal =
\{C_1, \cdots, C_k\}$ of a dataset $\Xcal = \{x_i\}_{i=1}^{N}$, such
that $\bigcup_{C \in \Ccal} C = \Xcal$ and $\forall C, C' \in \Ccal, C
\cap C' = \emptyset$. The partition is known as a \emph{clustering} of
$\Xcal$.

Most algorithms construct clusterings using pairwise similarities
among \records. But, pairwise similarities cannot capture many
complex relationships, e.g., \records $x_1$ and $x_2$ are similar when
clustered with \record $x_3$, but are otherwise dissimilar.  A natural
generalization that can capture these types of relationships are
similarities defined over sets of \records, which we refer to as
\hofs. Formally, a \hof is a function
$f: 2^\Xcal \times 2^\Xcal \rightarrow \mathbb{R}$.

Clustering with \hofs is ubiquitous, especially in \hac (from which
the name \hof is derived).  In \hac, many popular linkage functions
like single-, complete- and average-linkage are computed from pairwise
distance functions. More complex, set-wise \hofs are used in
applications such as image segmentation, within document coreference
and entity resolution; in the latter two domains, these functions are
often learned~\cite{clark2016improving, kohli2009robust,
  haghighi2010coreference, wiseman2016learning,
  zhang2013probabilistic}.  A unique capability of \hac is that it can
easily support an arbitrary \hof. This flexibility is essential to
combat the ill-posed nature of clustering.

\subsection{Model-based Separation}
\label{subsec:hsep}
Our goal is to design an algorithm that, like \hac, can support
arbitrary \hofs, but is dramatically faster.
% be used to produce clusterings according to any \hof.
In developing clustering algorithms, it is often useful to consider
various assumptions about the \emph{separability} of the underlying
data.  For example, in the pairwise setting one of the strongest data
assumptions is known as \emph{strict
  separation}~\cite{balcan2008discriminative}. This assumption holds
that any \record in ground-truth cluster $C_i$ is more similar to
every other \record in $C_i$ than any \record from a different
ground-truth cluster, $C_j$. It is easy to see that popular
instantiations of \hac (e.g., single-, average- and complete-linkage)
provably succeed under strict separation, which provides some
theoretical motivation for these algorithms.

We introduce a notion of \emph{model-based separation} for clustering
with a \hof. Since \hofs may operate on data of any type, we formalize
the definition in terms of a graph, where the \records correspond to
vertices.

\begin{definition}[Model-based Separation]
\label{def:hsep}
Let $G = (\Xcal, E)$ be a graph. Let $f: 2^\Xcal \times 2^\Xcal
\rightarrow \mathbb{R}$ be a \hof that computes the similarity of two
groups of vertices and let $\phi: 2^\Xcal \times 2^\Xcal \rightarrow
\{0, 1\}$ be a function that returns 1 if the union of its arguments
is a connected subgraph of $G$. Then $f$ separates $G$ if
\begin{align*}
  \forall s_0, s_1, s_2 \subseteq \Xcal, \quad \phi(s_0, s_1) > \phi(s_0,
  s_2) \Rightarrow f(s_0, s_1) > f(s_0, s_2)
\end{align*}
\end{definition}

\noindent In words, for a \hof $f$ to separate a graph $G$, take any
two sets of vertices, $s_0$ and $s_1$, such that $s_0 \cup s_1$ is
connected in $G$, i.e., $\phi(s_0, s_1) = 1$.  Then, for any set $s_2$
such $\phi(s_0, s_2) = 0$, the score of $f$ on input $(s_0, s_1)$ must
be greater than on input $(s_0, s_2)$.

Model-based separation offers a non-standard view of
clustering. Specifically, the \records of a dataset are treated as
vertices in a graph with latent edges. The ground-truth clusters are
the connected components of the graph and the goal of clustering is
to discover these components using a \hof.

We provide the following two examples to help build intuition about
model-based separation. The examples are used throughout the remainder
of our discussion.

\begin{example}[Clique]
  \label{ex:clique}
  Consider a graph $G=(\Xcal, E)$ in which each connected component is
  a clique.  Then if $f$ separates $G$, every vertex in a connected
  component, $C_i$, is more similar to all other vertices in $C_i$
  than any vertex in connected component $C_j$, where similarity is
  defined by $f$.
\end{example}
\noindent Thus, clique-structured connected components exactly capture strict
separation.

\begin{example}[Chain]
  \label{ex:chain}
  Consider a graph $G=(\Xcal, E)$ in which each connected component is
  chain-structured.  According to Definition \ref{def:hsep}, two
  vertices that are part of the same chain but do not share an edge
  may be dissimilar under $f$ even if $f$ separates $G$. However,
  any two segments of the chain connected by an edge are
  similar under $f$.
\end{example}
\begin{figure}[t]
  \captionsetup[subfigure]{justification=centering}
  \centering
  \begin{subfigure}[h]{0.4\columnwidth}
  \centerline{\includegraphics[width=\columnwidth]{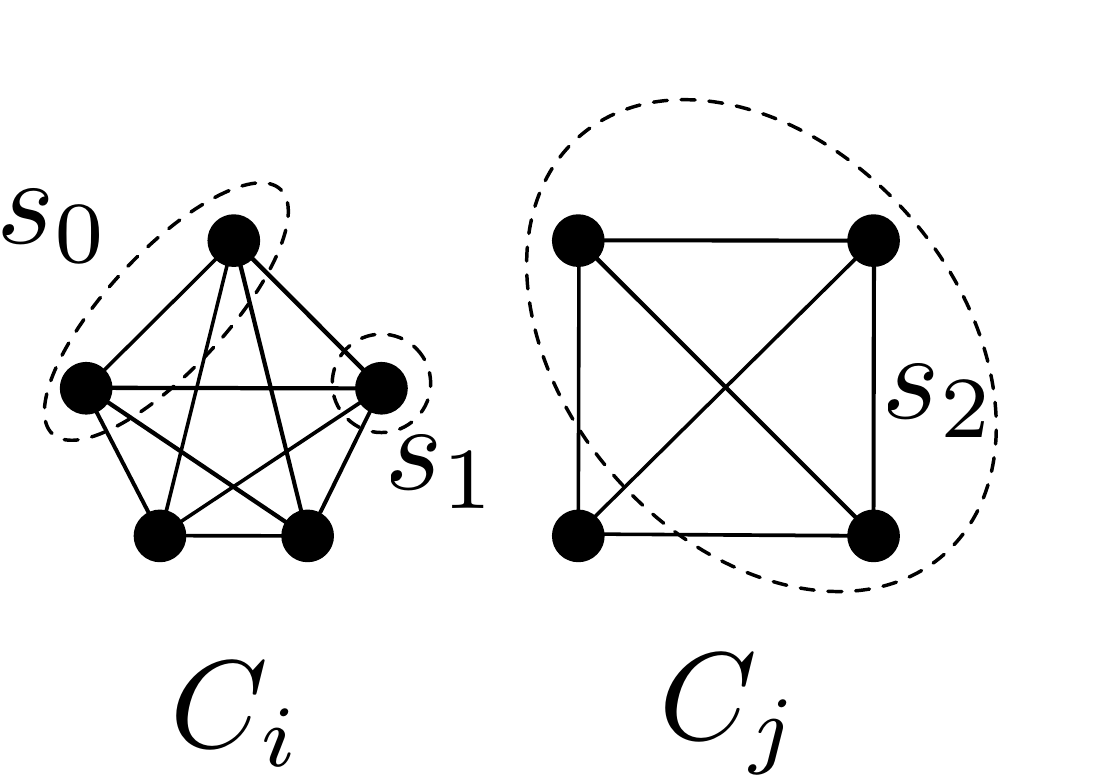}}
  \caption{Clique-shaped clusters.}
  \label{fig:clique-ex}
\end{subfigure}
\begin{subfigure}[h]{0.4\columnwidth}
  \centerline{\includegraphics[width=1.0\columnwidth]{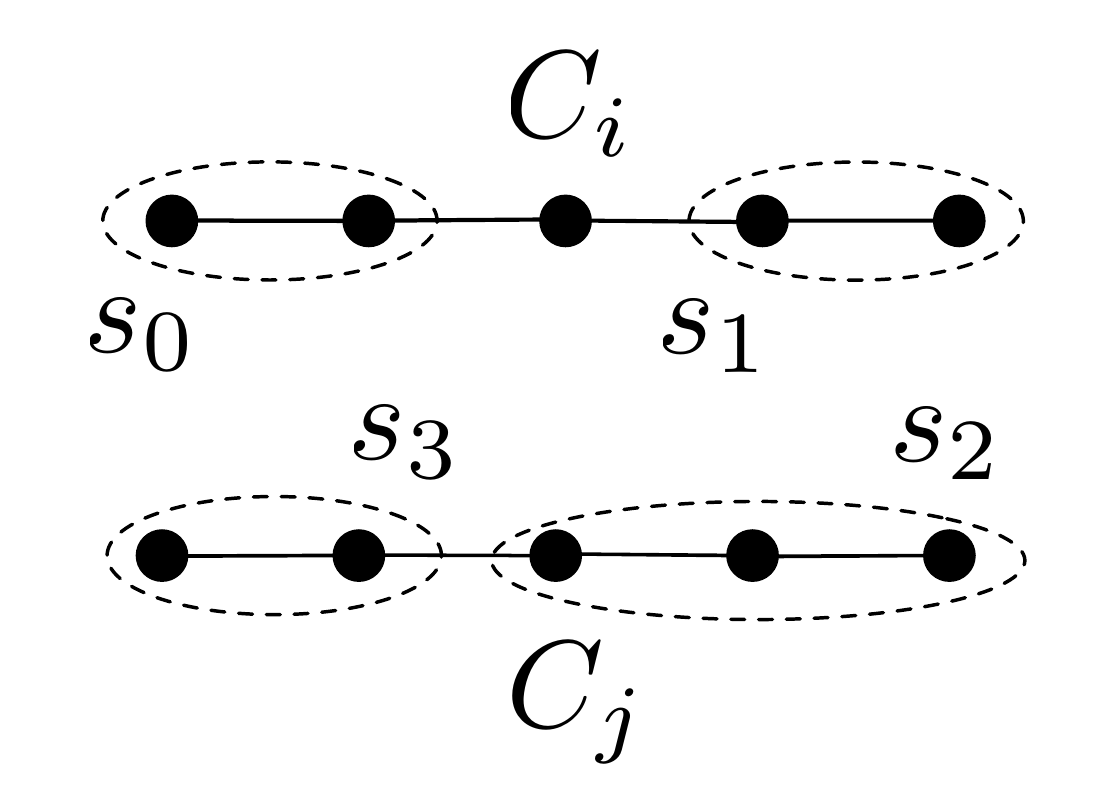}}
  \caption{Chain-shaped clusters.}
  \label{fig:chain-ex}
\end{subfigure}
\caption{\emph{Model-based separation}. Figure \ref{fig:clique-ex}
  shows two clique-shaped clusters with \records as vertices in a
  graph. If $f$ separates the graph then
  $f(s_0, s_1) > \max[f(s_0, s_2), f(s_1, s_2)]$ because $s_0$ and
  $s_1$ form a connected subgraph. In Figure \ref{fig:chain-ex}, even if
  $f$ separates the graph, it is possible for
  $f(s_0, s_1) < f(s_1, s_2)$. However, $f(s_1, s_2) < f(s_2, s_3)$.}
  \label{fig:clique-and-chain}
%%\vskip -0.2in
\end{figure}

\noindent A visual illustration of both clique and chain style clusters is
depicted in Figure \ref{fig:clique-and-chain}. As we will see, chain
structured connected components pose a challenge to existing
incremental algorithms, something we resolve with \alg (Section \ref{sec:inference}).

\subsection{Cluster Trees}
\label{subsec:clustertrees}
In most clustering problems, the appropriate number of clusters is
unknown a priori. \hac addresses this uncertainty by building a
\emph{cluster tree} over \records.

\begin{definition}[Cluster tree~\cite{krishnamurthy2012efficient}]
  A binary
\textbf{cluster tree} $\mathcal{T}$ on a dataset $\Xcal=\{x_i\}_{i=1}^N$ is a
collection of subsets such that $C_0 = \{x_i\}_{i=1}^N \in \mathcal{T}$
and for each $C_i,C_j \in \mathcal{T}$ either $C_i \subset C_j$, $C_j \subset
C_i$ or $C_i \cap C_j = \emptyset$.  For any $C \in \mathcal{T}$, if $\exists C'
\in \mathcal{T}$ with $C' \subset C$, then there exists two $C_L,C_R\in
\mathcal{T}$ that partition $C$.
\end{definition}

\noindent Given a cluster tree, $\Tcal$, any set of disjoint subtrees
whose leaves cover $\Xcal$ represents a valid clustering and is
referred to as a \emph{tree consistent
  partition}~\cite{heller2005bayesian}. Thus, cluster trees compactly
encode multiple alternative clusterings, allowing for a clustering to
be selected as a post-processing step.  Another advantage of using
cluster trees is that they often facilitate efficient search and
naturally group similar \records near one another in the hierarchy

We relate model-based separation, cluster trees and \hac in the following fact:

\begin{fact}
  \label{fact:hac}
  Let $f$ be a \hof that separates $G$. Then running \hac under
  $f$ returns a cluster tree, $\Tcal$, such that the connected
  components of $G$ are a tree-consistent partition of $\Tcal$.
\end{fact}
\noindent To see why, notice that in each iteration of HAC, the highest scoring
pair of remaining subtrees is merged. Since $f$ separates $G$, a
merger resulting in a subtree that corresponds to a connected subgraph
of $G$ has higher score than any merger resulting in a disconnected
subgraph of $G$. Even though HAC can construct a cluster tree that
contains the ground-truth clustering as a tree-consistent partition,
the algorithm costs $O(n^2\log n)$ for general linkage functions and
does not scale to large datasets. We will verify this claim
empirically in our experiments (Section \ref{sec:exp}).
\section{Rotations, Grafting and Grinch}
\label{sec:inference}
In this section we derive an efficient, incremental algorithm called
\alg that can be used to construct clusterings under any \hof. Like
\hac, the backbone of \alg is a cluster tree. We begin the discussion
by analyzing a greedy, incremental variant of \hac and when it
fails. Then, we introduce two subroutines, \rotate and \graft, that
can be used to enhance robustness. Finally, we present our algorithm,
\alg.

\subsection{Online HAC and Rotations}
\label{subsec:greedy}
An efficient alternative to \hac is its online variant that merges
each incoming \record with its nearest neighbor seen so far
(\greedy). For now, let us consider the setting in which a nearest
neighbor is found using a \hof, $f$.  Let $f$ separate a graph $G$
and let ground-truth clusters be cliques in $G$ (i.e., the data is
strictly separated). Even in this simple case, \greedy may construct a
cluster tree in which the ground-truth clustering is not a tree
consistent partition.  To see why, consider a stream in which the
first two \records, $x_1$ and $x_2$, are of the same ground-truth
cluster and the third \record, $x_3$ is of a different
cluster. Assume, without loss of generality, that \greedy adds $x_3$
as a sibling of $x_1$. Then the ground-truth clustering is not a tree
consistent partition of the resulting tree (and all subsequent trees).

To recover from such mistakes, local tree rearrangements may be
applied. Previous work uses \emph{rotations}, which swap a child and
its aunt in the tree, to correct local errors induced by unfavorable
arrival order~\cite{kobren2017hierarchical}. While originally designed
to be used with pairwise distances, the condition under which rotations
should be applied can be extended to \hofs:
\begin{align}
  \label{eq:rotate}
  f(v,\sib{v}) < f(v, \aunt{v})
\end{align}
where the functions $\sib{\cdot}$ and $\aunt{\cdot}$ return the
sibling and aunt of their input, respectively. In words, if a node
$v \in \Tcal$ achieves a higher score under $f$ with its aunt than
with its sibling, then the aunt and sibling should be swapped. Now,
let us revisit the example above. Since $x_1$ and $x_2$ are both
vertices in the same clique in $G$, they are connected by an
edge. Then, by model-based separation, $f(x_1, x_2) > f(x_1, x_3)$, so
a rotation will be applied, producing a tree that contains the
ground-truth clustering.

\begin{figure*}[t]
\captionsetup[subfigure]{justification=centering}
\begin{subfigure}[h]{0.32\textwidth}
  \centerline{\includegraphics[width=1.0\columnwidth]{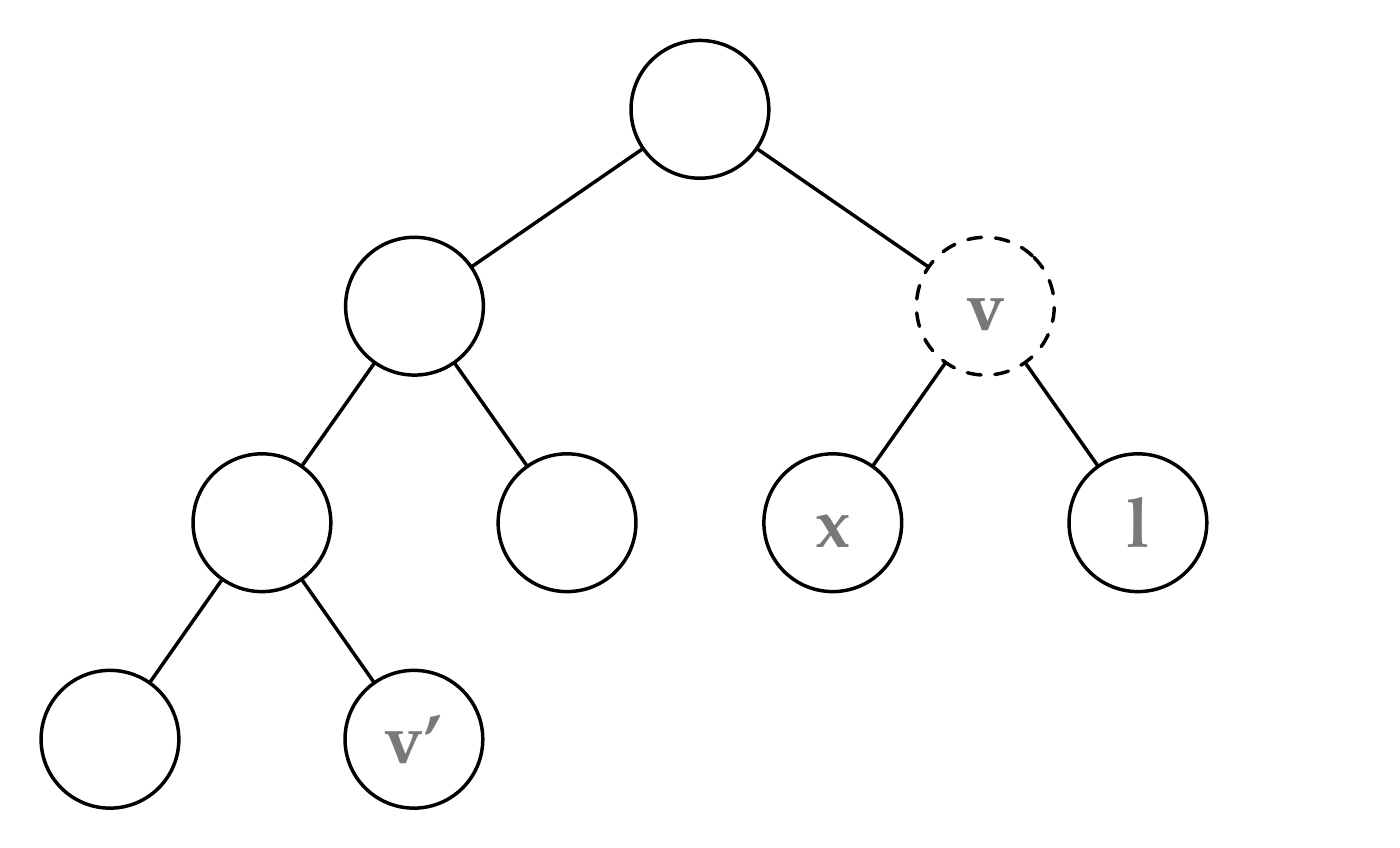}}
  \caption{$x$ added to cluster tree.}
\end{subfigure}
\begin{subfigure}[h]{0.32\textwidth}
  \centerline{\includegraphics[width=1.0\columnwidth]{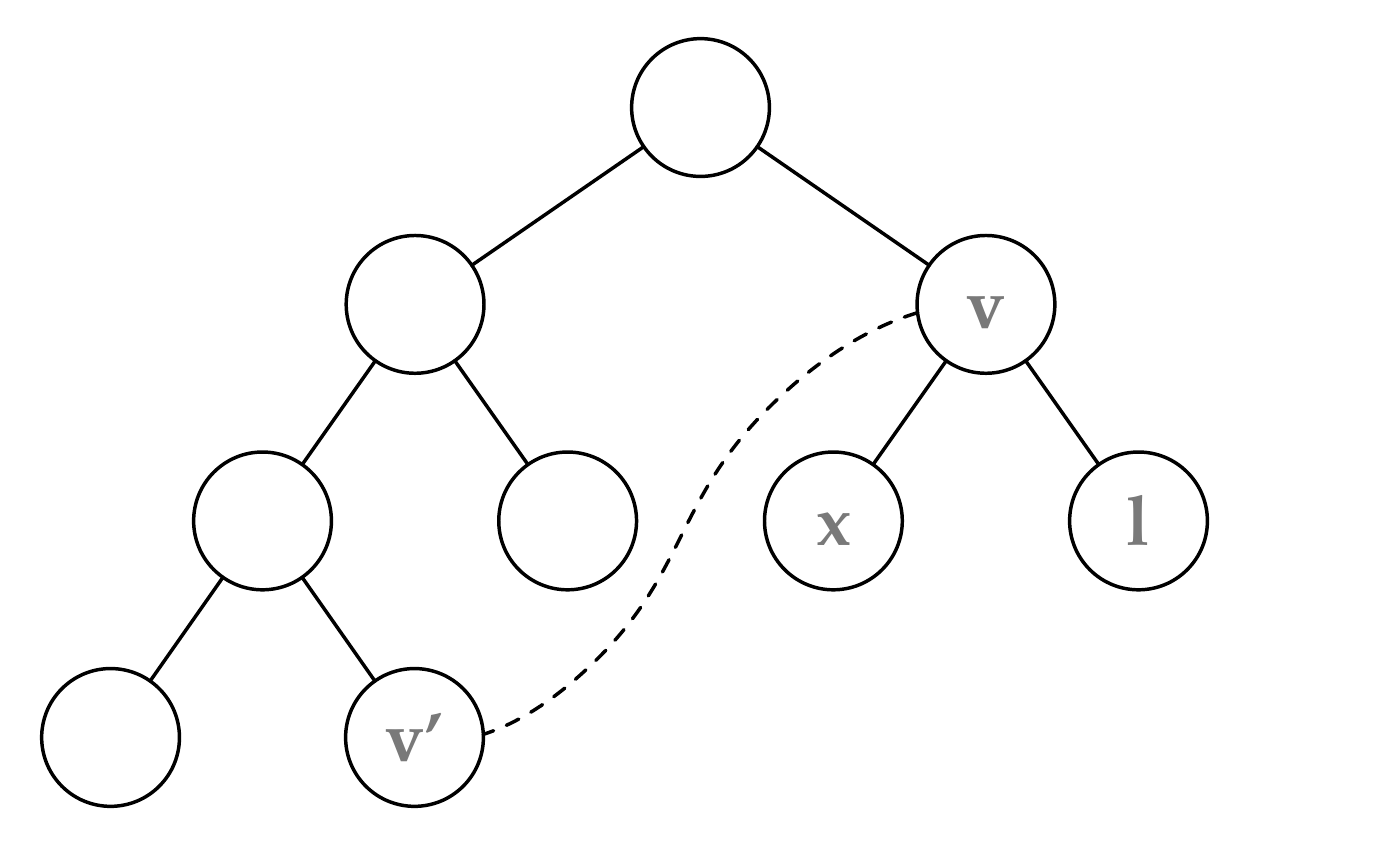}}
  \caption{$v$ finds nearest node $v'$.}
\end{subfigure}
\begin{subfigure}[h]{0.32\textwidth}
  \centerline{\includegraphics[width=1.0\columnwidth]{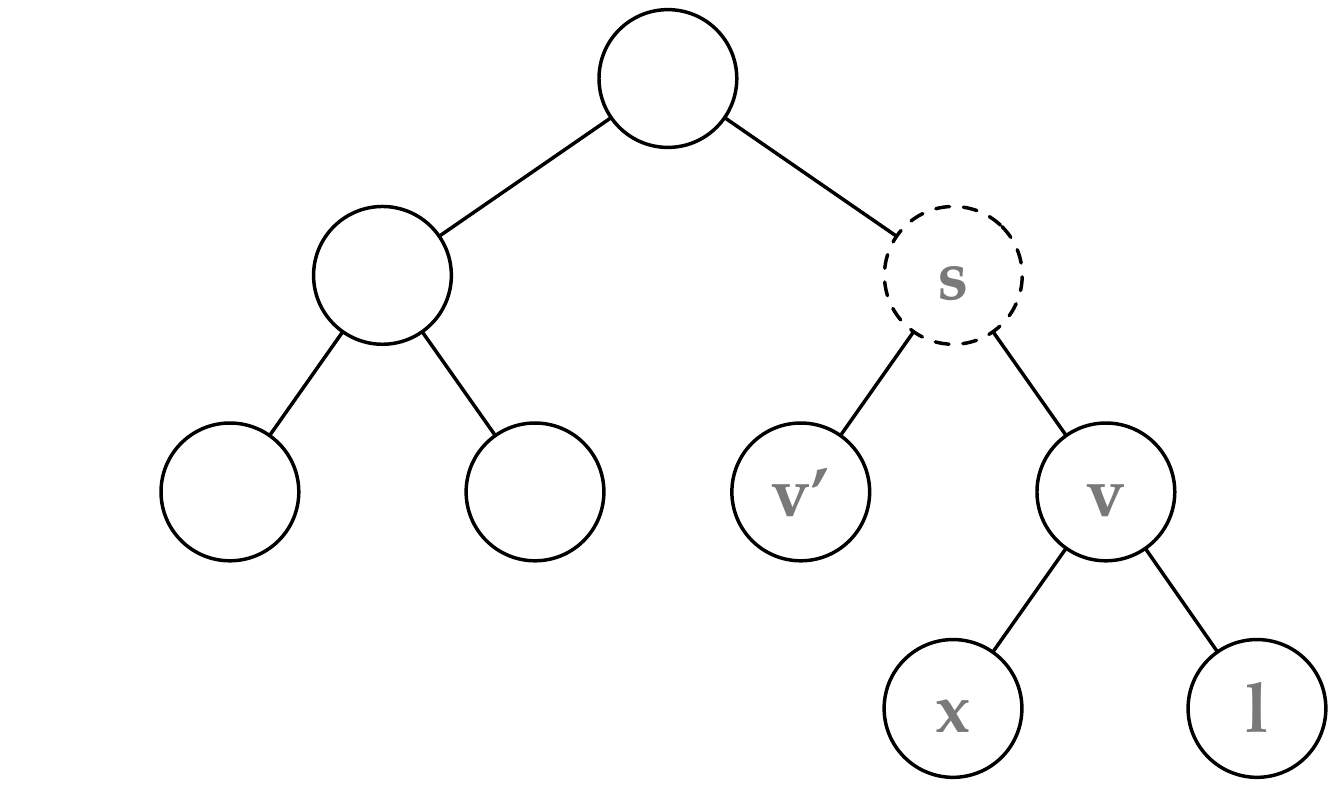}}
  \caption{$v$ is grafted to $v'$.}
\end{subfigure}
\caption{The \texttt{graft} subroutine. Dotted lines denote new
  nodes and mergers. Before $x$ is added to tree, $l$ and $v'$ reside
  in disjoint subtrees even though they belong to the same
  ground-truth cluster. The addition of $x$ creates the subtree with root
  $v$ and initiates the \texttt{graft} subroutine.}
\label{fig:graft}
%%\vskip -0.2in
\end{figure*}

Unfortunately, the \greedy algorithm, augmented with the ability to
performs rotations (\algrotate), cannot always recover the connected
components of a graph that is separated by $f$. In particular,
\algrotate cannot reliably recover chains (Example \ref{ex:chain}).
By virtue of being a local operation, rotations can only be used to
provably recover connected components that are clique-structure.

\subsection{Subtree Grafting}
\label{subsec:graft}
We introduce a non-local tree rearrangment called a \emph{graft},
which facilitates the discovery of chain-structured connected
components. At a high level, the \graft procedure with respect to a
node $v \in \Tcal$ searches $\Tcal$ for a node $v'$ that is both
similar to $v$ and dissimilar from its current sibling, $\sib{v'}$. If
such a subtree is found, $v'$ is disconnected from its parent and made
a sibling of $v$. A visual illustration of a successful \graft is
depicted in Figure \ref{fig:graft}.

In detail, a \graft searches the leaves of $\Tcal$ for the nearest
neighbor leaf of $v$ called $l$. Then it checks whether the following holds:
\begin{align}
  \label{eq:graft}
  \condition{v}{l}
\end{align}

\noindent i.e., $v$ and $l$ prefer each other to their current
siblings according to $f$. If the condition succeeds, merge $v$ and
$l$. If the condition fails because $l$ prefers its sibling to $v$,
retest the condition at $v$ and $l$'s parent, $\parent{l}$; if the
condition fails because $v$ prefers its sibling to $l$, then retest
the condition at $\parent{v}$ and $l$. Continue to check recursively
until the condition succeeds or until the first time two nodes, $v_1$
and $v_2$, are reached such that one is the ancestor of the
other. Pseudocode for the \graft subroutine can be found in
Algorithm~\ref{alg:graft}.  In the algorithm, \texttt{par} returns the
parent of a node in the tree, \texttt{lca} returns the lowest common
ancestors of its arguments and \texttt{makeSib} merges its arguments
and returns their new parent. \texttt{NN} performs a nearest neighbor
search and \texttt{constrNN} performs a nearest neighbor search that
excludes its second argument from the result.

\subsection{Tree Restructuring}
\label{subsec:restruct}
While the \graft subroutine facilitates discovery of chain-structured
clusters, poorly structured trees are susceptible to having the \graft
subroutine disconnect previously discovered ground-truth clusters.  As
an example, consider Figure \ref{fig:bad-struct}, in which $\lvs{v}$
form the connected subgraph $C_i$ (i.e., they all belong to the same
ground-truth cluster). Consider $v$'s left child, $v.l$, and its
descendants, which form a \emph{disconnected} subgraph. An attempt to
\graft either descendent, $x_1$ or $x_2$, may succeed, even when
initiated from a node (not depicted) whose descendants are not
connected to $C_i$. After such a \graft, $\Tcal$ cannot contain a
tree-consistent partition that matches the ground-truth clustering.

\begin{figure}[t]
  \captionsetup[subfigure]{justification=centering}
  \centering
\begin{subfigure}[h]{0.7\columnwidth}
  \centerline{\includegraphics[width=0.7\columnwidth]{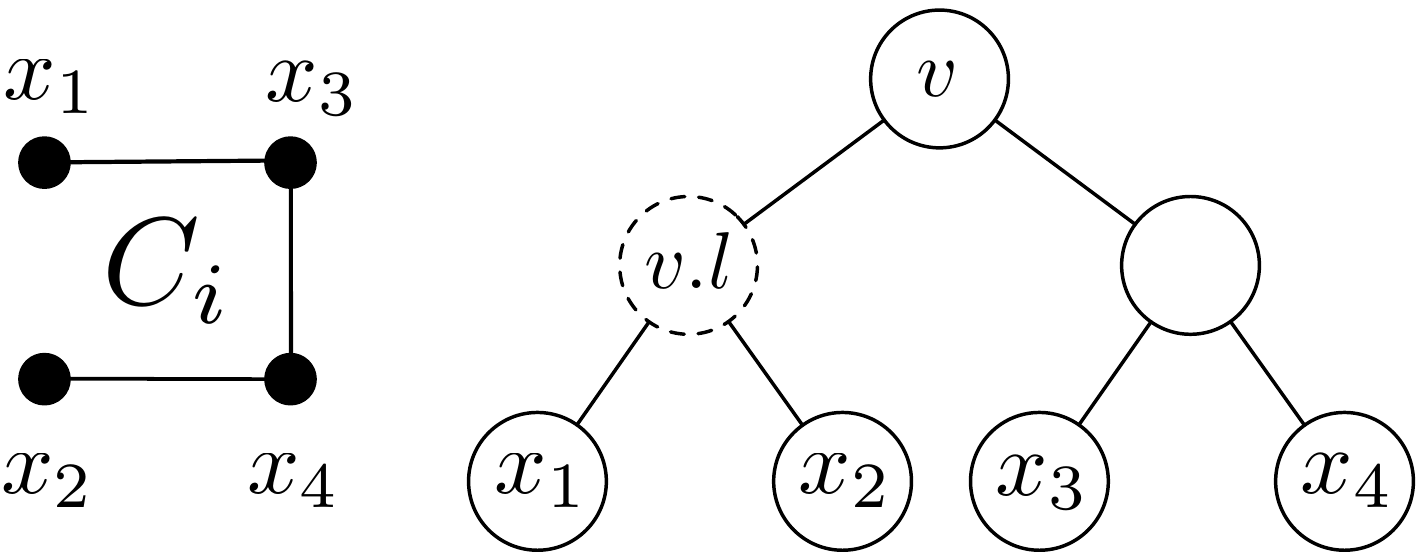}}
\end{subfigure}
\caption{\emph{Poorly structured tree}. Even though $v$'s leaves form
  a connected subgraph of the graph on the left of the Figure (i.e.,
  they all belong to cluster $C_i$), $v.l$'s descendent leaves, $x_1$
  and $x_2$, are disconnected. An attempt to \texttt{graft} either
  $x_1$ or $x_2$ from a node whose descendants are not members of
  $C_i$ may succeed.}
\label{fig:bad-struct}
\end{figure}

\begin{algorithm}[t!]
   \caption{\graft$(v , \Tcal, f)$}
   \label{alg:graft}
\begin{algorithmic}
  \STATE $l = \conNN (v, \lvs{v}, f, \Tcal)$
  \STATE $v'= \lca{v}{l}\texttt{; st = v}$
  \WHILE{$v \neq v' \wedge l \neq v' \wedge \sib{v} \neq l$}
    \IF{$f(v, l) > \max[f(v, \sib{v}, f(l, \sib{l})]$}
      \STATE $z = \sib{v} \texttt{; }v\, = \merge{v}{l}$
      \STATE $\texttt{restruct}(z, \lca{z}{v}, f)$
      \STATE $\texttt{break}$
    \ENDIF
    \IF{$f(v, l)\, <\, f(l, \sib{l})$}
      \STATE $l\, = \parent{l}$
    \ENDIF
    \IF{$f(v, l)\, <\, f(v, \sib{v})$}
      \STATE $v\, = \parent{v}$
    \ENDIF
   \ENDWHILE
    \IF{$v == st$}
      \STATE {\bfseries Output:} $v'$
    \ELSE
      \STATE {\bfseries Output:} $v$
    \ENDIF
\end{algorithmic}
\end{algorithm}

Notice that a subtree can defend against spurious \grafts by ensuring
that each of its descendant subtrees is connected. For example, in
Figure \ref{fig:bad-struct}, if $x_2$ and $x_3$ were swapped, then
each descendant subtree of $v$ would be connected. Moreover, after
such a swap, \grafts from nodes whose descendants were not part of
$C_i$ would necessarily fail (assuming that $f$ separates the
graph).

During tree construction, the only step that can result in a connected
subtree with disconnected descendants is the \graft subroutine (a
rigorous proof is included in the supplement). We introduce the \rst
(\emph{restructure}) subroutine, which is performed after a successful
\graft, and reorganizes a subtree with the intent of making each of
its descendants connected. Let $v'$ be a node that was just grafted,
$v$ be the previous sibling of $v'$ (i.e., before the graft) and let
$r=\lca{v}{v'}$ be the current least common ancestor of $v$ and
$v'$. \rst is initiated from $v$. First, the siblings of the ancestors
of $v$ (until $r$) are collected. Then, we find the node in the
collection most similar to $v$. If that node is more similar to $v$
than $v$'s current sibling (according to $f$), the two are
swapped. The intuition here is that if a \graft left $v$ and its new
sibling disconnected, then the swap serves as a mechanism to restore
the connectedness of $v$'s parent. Such swaps are attempted from the
ancestors of $v$ until $r$.  Pseudocode appears in
Algorithm~\ref{alg:restruct}.

\subsection{Grinch}
\label{sec:implementation}
Using the \rotate, \graft and \rst tree rearrangement routines discussed in
Section \ref{sec:inference}, we derive a new algorithm called \alg,
which stands for: \textbf{G}rafting and \textbf{R}otation-based
\textbf{INC}remental \textbf{H}iearchical clustering.  The steps of
the algorithm are as follows: when a new record, $x_i$, arrives, find
$x_i$'s nearest neighbor, $l$, among the leaves of $\Tcal$. Add $x_i$
to $\Tcal$ as a sibling of $l$.  Then, apply the \rotate subroutine
while Equation \ref{eq:rotate} is true. Finally, attempt to \graft
recursively from each ancestor of $x_i$.  Each time a \graft is
successful, restructure the tree to group similar items
together. Pseudocode for \alg can be found in Algorithm \ref{fig:alg}.

\begin{figure}
\begin{algorithm}[H]
   \caption{\texttt{restruct}$(z, r, f)$}
   \label{alg:restruct}
\begin{algorithmic}
  \WHILE{$z != r$}
  \STATE $as = \{\sib{a}\, \texttt{for}\, a \in \ancs{z} \backslash \ancs{r}\}$
  \STATE $m = \argmax_{a \in as} f(z, a)$
  \IF{$f(z, \sib{z})\, <\, f(z, m)$}
    \STATE $\texttt{swap}(\sib{z}, m)$
  \ENDIF
  \STATE $z = \parent{z}$
  \ENDWHILE
\end{algorithmic}
\end{algorithm}
\end{figure}
\begin{figure}
\begin{algorithm}[H]
   \caption{\texttt{Insert}$(x_i , \Tcal, f)$}
   \label{fig:alg}
\begin{algorithmic}
   \STATE $l = $\texttt{ NN}$(x_i, f, \Tcal) \texttt{; } t = \texttt{ makeSib}(x_i, l)$
   \WHILE{$f(x_i, \sib{x_i}) < f(\aunt{x_i}, \sib{x_i})$}
   \STATE \rotate$(x_i, \aunt{x_i})$
   \ENDWHILE
   \STATE $p = \parent{x_i}$
   \WHILE{$p \ne \texttt{null}$}
   \STATE $\texttt{curr} = \graft(p, \Tcal, f)$
   \ENDWHILE
\end{algorithmic}
\end{algorithm}
\end{figure}

\begin{theorem}
\label{thm:hsep}
Let $\Xcal = \{x_i\}_{i=1}^{N}$ be a dataset with ground-truth
clustering $\cstar = \{C_1, \cdots, C_k\}$. Let $f$ separate a graph
$G$ on vertices $\Xcal$ and let each cluster $C \in \cstar$ be a
connected component in $G$.  Then \alg recovers a cluster tree such
that $\cstar$ is a tree consistent partition of $\Tcal$ regardless of
the input order.
\end{theorem}

The proof of Theorem \ref{thm:hsep} can be found in the appendix.
\section{Experiments}
\label{sec:exp}
We experiment with \alg to assess its scalability and accuracy. We
begin by demonstrating that \alg outperforms other incremental
clustering algorithms on a synthetic dataset. Observing that some of
the steps of \alg are underutilized, we present 4 approximations of
\alg's algorithmic components. We apply each approximation in turn and
show that together they dramatically improve \alg's scalability
without compromising its clustering quality. Then, we compare the
approximate variant of \alg to state-of-the-art large scale
hierarchical clustering methods. To showcase the flexibility of \alg,
we also provide experimental results in entity resolution, where the
linkage function is learned.  Finally, we provide analysis of the
\graft subroutine--\alg's distinguishing feature--and perform
experiments to demonstrate the algorithm's robustness.

\paragraph{Dendrogram Purity}
Before beginning, we briefly review \emph{dendrogram purity}, a
preferred method of holistically evaluating hierarchical
clusterings~\cite{blundell2011discovering, heller2005bayesian,
  kobren2017hierarchical}. Dendrogram purity is computed as follows:
Let $\cstar = \{C_1, \cdots, C_k\}$ be the ground-truth clustering of
a dataset $\Xcal$, and let
$\mathcal{P}^\star = \{(x,x') | x,x' \in \Xcal,\ \cstar(x) =
\cstar(x') \}$ be the set of all \record pairs that belong to the same
ground-truth clusters.  Then the dendrogram purity (DP) of a cluster
tree, $\Tcal$ is:
\begin{align*}
\textrm{DP}(\Tcal) = \frac{1}{|\mathcal{P}^\star|} \sum_{(x,x') \in
\mathcal{P}^\star} \pur{\lvs{\lca{x}{x'}}, \cstar(x)}
\end{align*}
where \lca{x}{x'} returns the least common ancestor of $x$ and $x'$ in
$\Tcal$, \lvs{\cdot} returns the descendant leaves of its argument,
and $\pur{\cdot, \cstar(x)}$ takes a collection of leaves and computes
the fraction that belong to ground-truth cluster $\cstar(x)$.

\subsection{Synthetic Data Experiment}
In our first experiment, we compare \alg to other incremental
hierarchical clustering algorithms on a synthetic dataset in order to
begin to understand \alg's empirical performance characteristics in a
controlled manner. The data is generated so that it satisfies
model-based separation with respect to cosine similarity. In
particular, the dataset contains 2500 10000-dimensional binary vectors
that belong to 100 clusters, with 25 points per cluster. Points in
cluster $k$ have bits $100k$ to $100(k-1)$ set randomly to 1 with
probability $0.1$. All other bits are set to 0. This way, across
cluster points have cosine similarity 0 and within cluster points can
have either 0 or non-zero cosine similarity. In other words, two
points, $x_1$ and $x_2$, in the same cluster can appear to be
dissimilar and end up in distant regions of the tree. The
representation of each internal node in the \alg tree is the sum of
the vectors of its descendent leaves. Thus, compute the cosine
similarity between two nodes $v$ and $v'$ as the cosine similarity
between their aggregated vectors (
we refer to this as \emph{cosine linkage} in the following sections
). We compare \alg, \algrotate and
\greedy.

The experimental results reveal that \alg achieves perfect dendrogram
purity (1.0), which is expected given \alg's correctness guarantee.
\algrotate achieves a dendrogram purity of 0.872 while \greedy
achieves 0.854. \algrotate and \greedy do not construct trees of
perfect purity because of their inability to globally rearrange a
cluster hierarchy.

\subsection{Approximations}
Some of the algorithmic steps of \alg, which are required to prove its
correctness, are seldom invoked in practice. For example, and perhaps
expectedly, a \graft is unlikely to succeed between two nodes close to
the root of the tree. Therefore, we introduce handful of
approximations designed to have little effect on the quality of the
clusterings constructed by \alg, but also designed to make the
algorithm significantly faster in practice.

\begin{enumerate}[noitemsep,topsep=0pt,parsep=0pt,partopsep=0pt,leftmargin=*]
\item \textbf{Capping.} Recursive subroutines like \graft and \rotate improve performance,
but they are also computationally expensive to check, and often
fail. Moreover, we notice that tree rearrangements that occur close to
the root do not have a significant, instantaneous effect on dendrogram
purity. Therefore, we introduce \emph{rotation}, \emph{graft} and
\emph{restructure caps}, which prohibit rotations, grafts and
restructures from occurring above a height, $h$.
\item \textbf{Single Elimination Mode.}
The \graft subroutine generally improves \alg's clustering
performance, and is essential in attaining perfect purity on the
synthetic dataset, but we find that \graft attempts are rejected many
more times than they are accepted. However, at times, we observe that
a sequence of recursive \grafts are accepted when initiated close to
the leaves. Therefore, to limit the number of attempted \grafts while
retaining these \graft sequences, we introduce \emph{single
  elimination mode}. In this mode, the recursive grafting procedure
terminates after a \graft between $v$ and $v'$ fails because both
prefer their current siblings to a merge.
\item \textbf{Single Nearest Neighbor Searching.}
\alg makes heavy use of nearest neighbor search under the linkage
function $f$. Rather than perform nearest neighbor search anew for
each \graft, when a \record arrives, we perform a single $k$-NN search
($k \in [25, 50]$) and only consider these nodes during subsequent
\grafts (until the next \record arrives).
\item \textbf{Navigable Small World Graphs.}
Instead of performing nearest neighbor computations exactly, we can
perform them approximately.  To this end, we employ a \emph{navigable
  small world} nearest neighbor graph (NSW)--a data structure inspired
by decentralized search in small world
networks~\cite{watts1998collective, kleinberg2000small,
  kleinberg2006complex}. To find the nearest neighbor of a
\record, $x_i$, in an NSW, begin at a random node, $v$.  If the similarity
between $x_i$ and $v$ is maximal among all neighbors of $v$,
terminate; otherwise, move to the neighbor of $v$ most similar to
$x_i$. To insert a new \record, $x_j$, find its $k$ nearest neighbors
and add edges between those neighbors and a new
\record~\cite{malkov2014approximate}. Thus, NSWs are constructed
online. In practice, we simultaneous construct a hierarchical
clustering and an NSW over the \records stored in the tree's leaves.
\end{enumerate}

\begin{table}
  \centering
  \begin{tabular}{l c c c c c}
    \textbf{ALOI}\\
    \hline
    % & \multicolumn{5}{c}{ALOI} \\ \cline{2-6}\\
    \textbf{Approx.}    & \textbf{DP} & \textbf{Time (s) } & \textbf{\# Rotate} & \textbf{\# Graft} & \textbf{\# Restr.} \\
    \alg (No Approx).  & 0.533       &	85.371             &	7107             &	2435          &	1088             \\
\hspace{.2cm} w/ Cap (100) & 0.533 &	48.452 & 	6495 &	2157 &	686 \\
\hspace{.2cm} w/ Single Elimn & 0.534 &	39.019 &	6574 &	1586 &	533 \\
\hspace{.2cm} w/ Single NN &0.540 &	22.226 &	6441 &	1516 &	570 \\
\hspace{.2cm} w/ no Restruct &0.538 &	14.292 &	6477 &	1634 &	0   \\
\hspace{.2cm} w/ no Graft & 0.506 &	12.748 &	6747 &	0 &	0   \\
\hspace{.2cm} w/ no Rotate & 0.442 &	14.793 &	0 &	0 &	0   \\
    \\
    \textbf{Synthetic}\\
    \hline
%    \multicolumn{5}{c}{Synthetic} \\ \cline{2-6} \\
    \textbf{Approx.} & \textbf{DP} & \textbf{Time (s) } & \textbf{\# Rotate} & \textbf{\# Graft} & \textbf{\# Restr.}\\
    \alg (No Approx).  & 1.0         & 160.307              & 2558                & 578   &  203\\
\hspace{.2cm} w/ Cap (100) & 0.993 & 164.328 & 2558 & 578 & 194\\
\hspace{.2cm} w/ Single Elimn & 0.997 & 157.622 & 2523 & 526 & 184\\
\hspace{.2cm} w/ Single NN & 0.993 & 83.014 & 2517 & 415 & 148\\
\hspace{.2cm} w/ no Restruct & 0.993 & 82.262 & 2476 & 426 &  0\\
\hspace{.2cm} w/ no Graft & 0.872 & 82.055 & 2259 & 0   &  0\\
\hspace{.2cm} w/ no Rotate & 0.854 & 80.526 & 0 & 0 & 0\\

    \hline
  \end{tabular}
\caption{\emph{Ablation.} Each row in the table represents \alg with
  the corresponding approximation applied in addition to all
  approximations contained in previous rows. The first 4
  approximations significantly decreases the computational cost of
  \alg, but do not compromise DP. The ablation is performed for the
  first 5000 points of ALOI and the Synthetic datasets.}
\label{fig:approx_tbl2}
\end{table}

To measure the effects of our approximations on the speed and quality
of the resulting algorithm, we conduct the following ablation. We run
\alg on our synthetically generated dataset as well as a random 5k
subset of the \textsc{ALOI}~\cite{geusebroek2005amsterdam} dataset and
measure dendrogram purity, time, and the number of calls made to
\rotate, \graft and \rst.  We repeat the procedure multiple
times, each time adding one of the following approximations, in order:
capping, single elimination, single nearest neighbor search and
approximate nearest neighbor search. Capping and is performed at
height 100.  We also experiment with removal of the \graft and \rotate
subroutines.

The result of the ablation is contained in
Table~\ref{fig:approx_tbl2}. We observe that, for both datasets, each
of the approximations reduces the computational cost of algorithm
without effecting the resulting DP. However, once \grafts are removed,
the DP drops by 3\% on ALOI and 12\% on the synthetic datasets. When
\rotate is also removed, DP drops by an additional 6\% and 2\%,
respectively.

Having verified that on a subset of ALOI our approximations improve
scalability at little expense in terms of dendrogram purity, in the
following experiments we report results for \alg in single elimination
mode and with the rotation cap set to $h=100$.

\subsection{Large Scale Clustering}
\label{subsec:xcluster}

\begin{table*}[t]
  \centering
  \scriptsize
	\begin{tabular}{c c c c c c c c c c}
		\hline
		\textbf{Alg. (link.)} & \textbf{CovType} & \textbf{ILSVRC12 (50k)}  & \textbf{ALOI} & \textbf{Speaker}   & \textbf{ImgNet (100k)}\\
		\hline
		\textbf{\alg} (Avg)  & 0.43 $\pm$ 0.00 & \bf 0.557 $\pm$ 0.003 & \bf 0.504 $\pm$ 0.002 &   0.480 $\pm$ 0.003 & \bf 0.065 $\pm$ 0.00\\
		\textbf{\alg} (CS)   & 0.43	$\pm$ 0.00 & \bf 0.544 $\pm$ 0.005 & \bf 0.499 $\pm$ 0.003 & \ 0.478 $\pm$ 0.003 &   0.062 $\pm$ 0.00\\
		\textbf{\algrotate} (Avg) & 0.43 $\pm$ 0.01 & 0.545 $\pm$ 0.004  &0.476 $\pm$ 0.004 & 0.407 $\pm$ 0.003 & 0.063 $\pm$ 0.00\\
		\textbf{\algrotate} (CS)  & 0.44 $\pm$ 0.01 & 0.513 $\pm$ 0.007 & 0.472 $\pm$ 0.003 & 0.406 $\pm$ 0.003 & 0.062 $\pm$ 0.00\\
		\textbf{\greedy}          & 0.44 $\pm$ 0.01 & 0.527 $\pm$ 0.00 & 0.435 $\pm$ 0.004 &  0.317 $\pm$ 0.002 & 0.0589\\
		\textbf{\perch} \cite{kobren2017hierarchical} & 0.45 $\pm$ 0.00 & 0.53 $\pm$ 0.003 &  0.44 $\pm$ 0.004 &   0.37 $\pm$ 0.002 & \bf 0.065$\pm$0.00\\
		\textbf{\perch}-BC \cite{kobren2017hierarchical}& 0.45 $\pm$ 0.00 & 0.36 $\pm$ 0.005 & 0.37 $\pm$ 0.008 & 0.09 $\pm$ 0.001 & 0.03 $\pm$ 0.00\\
		\textbf{\mbhac} (Best) \cite{kobren2017hierarchical} & 0.44 $\pm$ 0.01 & 0.43 $\pm$ 0.005 & 0.30 $\pm$ 0.002 & 0.01 $\pm$ 0.002 &  --- &\\
		\textbf{HAC} (Avg) \cite{kobren2017hierarchical} & -- & \bf 0.54 &  -- & \bf 0.55 &  --  \\
          \hline
	\end{tabular}
	\caption{Dendrogram Purity results for \alg and baseline
          methods. We compare two linkage functions: approximate
          average linkage (Avg) and cosine similarity linkage
          (CS). }
	\label{tab:xclusterdp}
\end{table*}

We compare \alg with the following 4 algorithms: \textbf{\greedy} - an
online hierarchical clustering algorithm that consumes one \record at
a time and places it as a sibling of its nearest neighbor;
\textbf{\algrotate} - an incremental algorithm that places a \record
next to its nearest neighbor and then performs rotations until
Equation \ref{eq:rotate} holds;
\textbf{\mbhac} - the mini-batch version of
\hac, which keeps a buffer of size $b$, runs a single step of \hac using
the \records in the buffer and then adds the next record to the buffer;
 \textbf{\hac} - best-first, bottom-up hierarchical agglomerative
clustering and \textbf{\perch} - a state-of-the-art large scale hierarchical clustering method.

We run each algorithm on 5 large scale clustering datasets: CovType,
a datset of forest covertype,
ALOI~\cite{geusebroek2005amsterdam}, a 50K subset of the Imagenet
ILSVRC12 dataset~\cite{russakovsky2015imagenet} and the Speaker
dataset~\cite{greenberg2014nist}, and a 100K subset of
ImageNet containing all 17K classes not just the subset
in ILSVRC12. Datasets have 500K, 50K, 100K, 36K, and 100K
instances, respectively. We run each \hac variant under two different
\hofs: average linkage and cosine linkage.  To compute the cosine
similarity between two nodes, $v$ and $v'$, first, for each node,
compute the sum of the vectors contained at their descendant
leaves. Then, compute the cosine similarity between the aggregated
vectors.

Results are displayed in Table~\ref{tab:xclusterdp}, where we record
the dendrogram purity averaged over 5 replicates of each algorithm,
where for each replicate we randomize the arrival order of the
data. %
The table reveals that \alg--under both linkage functions--outperforms
the corresponding versions of \algrotate and \greedy on all datasets
except for on the CovType dataset where the methods all seem to perform
equally well. This underscores
the power of the \graft subroutine. \alg with approximate nearest
neighbor search even outperforms \perch, which uses exact nearest
neighbor search, on ALOI.  Recall that, unlike the \hac variants,
\perch employs a specific linkage function. Seeing as the \hac
variants outperform \perch on Speaker suggests that the ability to
equip various \hofs can be advantageous. \hac is best on Speaker, but
cannot scale to ALOI.

\subsection{Author Coreference}
\label{subsec:coref}
Bibliographic databases, like PubMed, DBLP, and Google Scholar,
contain citation records that must be attributed to the corresponding
authors. For some records, the attribution process is easy, but for
many others, the identities of a publication's authors are ambiguous.
For example, DBLP contains hundreds of citations written by different
authors named ``Wei Wang'' that currently cannot be
disambiguated~\cite{dblp:wei}. Intuitively, author coreference
datasets often exhibit chain like structures because a single citation
written by a prolific author (perhaps in a short-lived collaboration)
may only be similar to a small number of that author's other citations
and dissimilar from the rest.

Following previous work, we train a \hof to predict the likelihood
that a group of citation records were all written by the same
author~\cite{culotta2007author, singh2011large,
  wick2012discriminative}.  We train our model by running \hac and, at
each step, use the model to predict the precision of merging two
groups of records. (A similar training technique was previously
proposed for entity and event coreference~\cite{lee2012joint}.) Our
model has access to features like: coauthor names and publication
title, venue, year, etc.

We compare the 5 \hac variants in author coreference on two datasets
with labeled author identities: \textbf{Rexa}~\cite{culotta2007author}
and \textbf{PSU-DBLP}~\cite{han2005name}.  As is standard in author
coreference we evaluate the methods using the pairwise F1-score of a
predicted flat clustering against the ground-truth clustering, which
is the harmonic mean of precision and recall.  To compute pairwise
F1-score, each pair of citations that appear in both the same
ground-truth and predicted clusters is considered a true positive; each
pair of citations that belong to different ground-truth clusters but
the same predicted cluster is considered a false positive. None of the
authors represented in the test set, have any publications in the
training set.

Figure~\ref{tab:acorefres} shows the precision, recall, and pairwise
F1-score achieved by each method. The results show that \alg
outperforms the other scalable methods on both datasets and even
outperforms \hac on DBLP. This behavior may stem from overfitting of
the learned \hof, which is exploited by \hac; since \alg only
approximates \hac, it can be thought of as a form of regularization.
Again, we observe that \alg outperforms \greedy and \algrotate on both
datasets underscoring the importance of the \rotate and \graft
procedures.

\begin{table*}[t]
	\centering
	\begin{tabular}{ c c c c c c c c c}
		\hline
	            & \multicolumn{3}{c}{Rexa}                  & \multicolumn{3}{c}{DBLP} \\
		\textbf{Algorithm} & \textbf{Pre} & \textbf{Rec}  & \textbf{F} &  \textbf{Pre} & \textbf{Rec}  & \textbf{F} \\
		\hline
		\textbf{\alg}           &  0.808 &	0.883 &	\bf 0.844 $\pm$ 0.004  & 0.809 & 0.620 & \bf 0.701 $\pm$ 0.013\\
		\textbf{\algrotate}      & 0.864 &	0.641 &	0.734 $\pm$ 0.057       & 0.876 & 0.554 & 0.678 $\pm$ 0.019 \\
		\textbf{\greedy}       & 0.850 &	0.209 &	0.331 $\pm$  0.094    & 0.827 &	0.151 &	0.255 $\pm$ 0.027  \\
		\textbf{\mbhac-Med.}         & 0.807 &	0.881 &	\bf 0.843 $\pm$ 0.0009   & 0.375 & 0.631 & 0.461 $\pm$	0.072   \\
		\textbf{\mbhac-Sm.}        & 0.922 &	0.333 &	0.483  $\pm$ 0.061   & 0.697 & 0.151 &0.247 $\pm$ 0.004  \\
		\hline
		\textbf{\exact}       & 0.805 &	0.887 &	\bf 0.844      &  0.741 &	0.600 &	0.664 \\
		\hline
	\end{tabular}
	\caption{Precision, recall and F-Score of various methods on the Rexa and DBLP datasets.}
	\label{tab:acorefres}
\end{table*}

\subsection{Significance of Grafting}
\label{subsec:sig-graft}
The results above indicate that \alg--even when employing a number
of approximations--constructs trees with higher dendrogram purity than
other scalable methods in a comparable amount of time. Interestingly,
\alg only differs from \rotate in its use of the \graft (and
subsequent \rst) subroutine. To better understand the
significance of \grafting, we compare \alg and \rotate on the
first 5000 points of ALOI.

\begin{figure*}[t]
	\captionsetup[subfigure]{justification=centering}
	\begin{subfigure}[h]{0.49\textwidth}
		\centerline{\includegraphics[width=\textwidth]{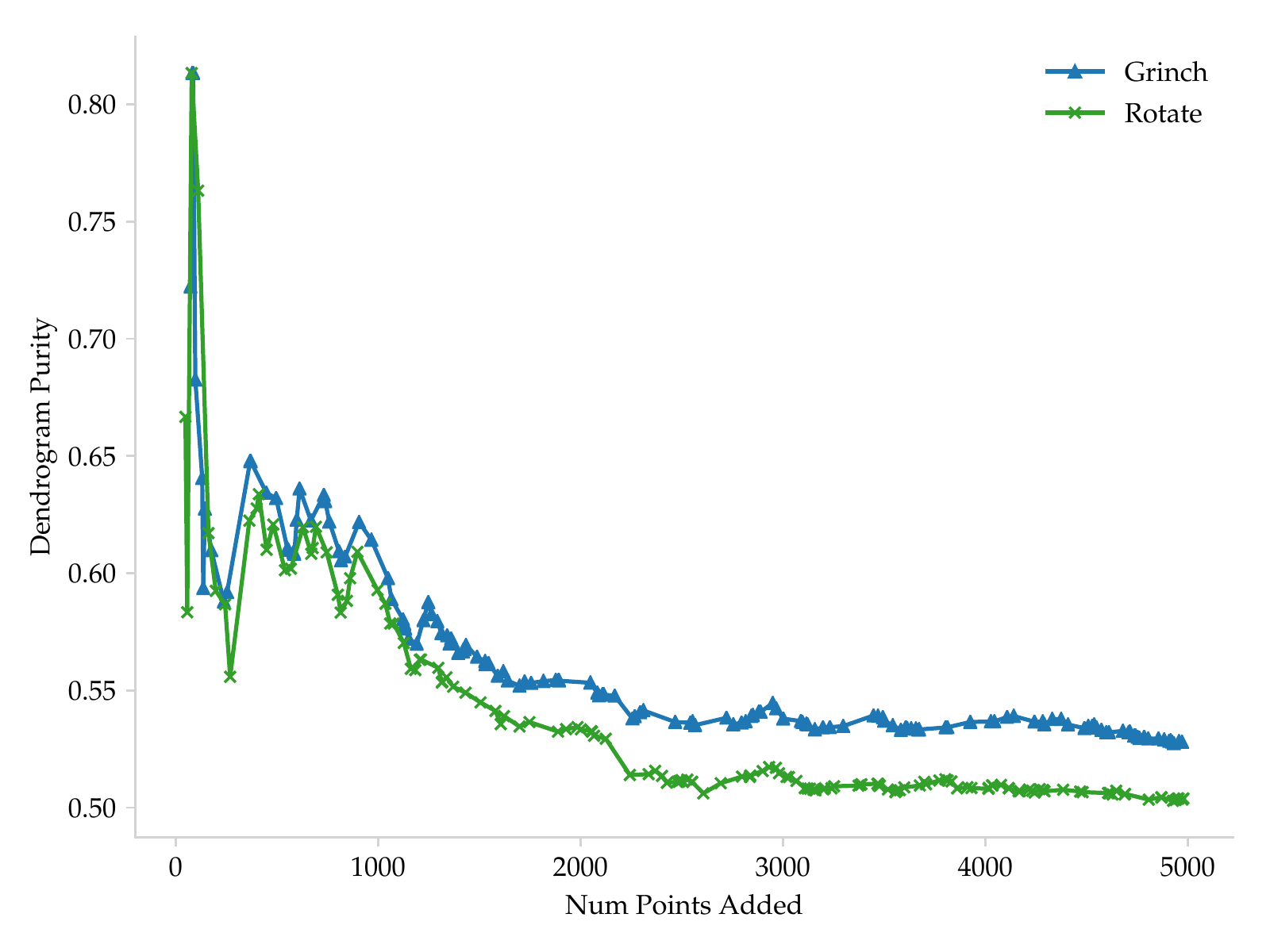}}
		\caption{Dendrogram purity per point.}
		\label{fig:cost}
	\end{subfigure}
	\begin{subfigure}[h]{0.49\textwidth}
		\centerline{\includegraphics[width=\textwidth]{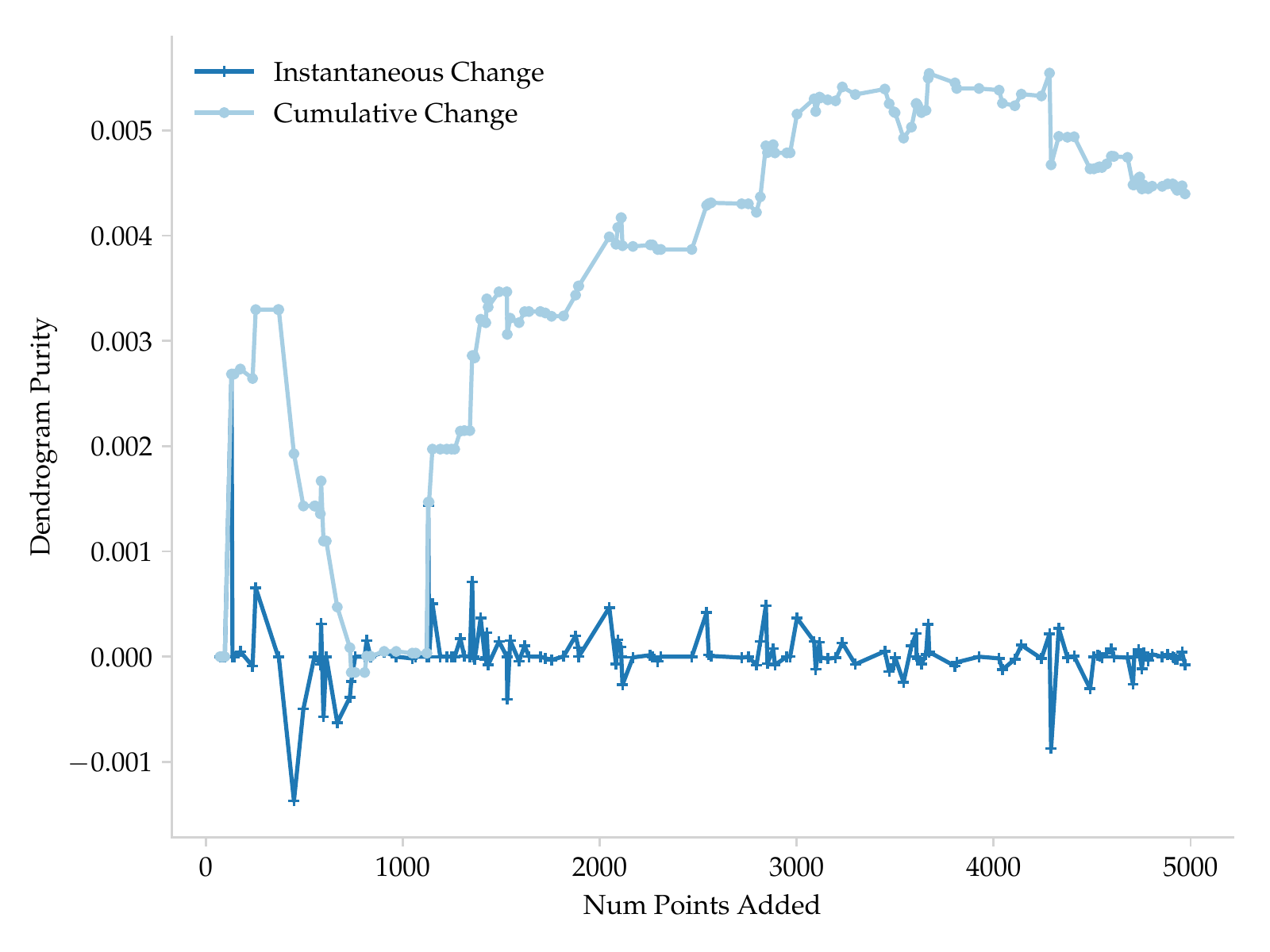}}
		\caption{Instant./cumulative change in DP due to \grafts.}
		\label{fig:grafts}
	\end{subfigure}
	\caption{Figure \ref{fig:cost} shows the dendrogram purity of
          two trees, one built by \alg and the other built by \rotate,
          on the first 5000 points of ALOI. The dendrogram purity of
          the tree built \alg is greater than that of the tree built
          by \rotate. Figure \ref{fig:grafts} plots the instantaneous
          and cumulative change in dendrogram purity due to
          \grafts. While \alg achieves ~3\% larger dendrogram purity
          than \rotate}
\end{figure*}

Figure \ref{fig:cost} shows that dendrogram purity as a function of
the number of data points inserted for both \alg and \rotate and the
first 5000 points of ALOI. Echoing the results above, by ~1000 points,
\alg dominates \rotate.

Figure \ref{fig:grafts} shows the instantaneous and cumulative change
in dendrogram purity due to \grafts made by \alg. That is, for the
$i$th data point, $x_i$, we record the dendrogram purity after $x_i$
is inserted and rotations are performed (i.e., what would be executed
by \rotate). Then, we perform \grafting (if appropriate) and record
the dendrogram purity \emph{after} all recursive \grafts have been
completed. The difference between the dendrogram purity after
\grafting and before \grafting (but after rotations) is the
instantaneous change in dendrogram purity due to \grafts; the sum of
instantaneous changes is the cumulative change.

Note the $y$-axis of Figure \ref{fig:grafts}, which reveals that even
the most instantaneously significant \grafts only lead to minute
changes in dendrogram purity (of about 0.001). Moreover, after 5000
points, the cumulative change in dendrogram purity due to \grafts is
less than 0.005--hardly accounting for the difference in dendrogram
purity between the tree built by \alg and the tree built by \rotate
(of ~0.03). We conclude from these measurements that the increase in
performance due to the \graft subroutine is related to the
rearrangement of small numbers of points. These rearrangements do not
immediately have significant impact on dendrogram purity, but they do
have significant long-term affects. To make this hypothesis more
concrete, consider the case in which two dissimilar data points from
the same cluster are split between two distant regions of the tree
early on in clustering. The points are never merged (via a \graft) and
so each point draws a significant portion of the cluster's other
\records to its location in the tree. This has dire consequences with
respect to dendrogram purity.  If a \graft is performed early on to
correct the split, an adverse scenario like this can be averted.

\subsection{Robustness}
\label{subsec:robust}
For completeness, we perform an experiment used in previous work to
test an incremental clustering algorithm's robustness to data point
arrival order~\cite{kobren2017hierarchical}. In the experiment, a
dataset is ordered in two specific ways:
\begin{description}
	\item[Round-Robin] Randomly determine an ordering of
          ground-truth clusters. Then, construct a data point arrival
          order such that the $i$th data point is a member of cluster $i$
          \texttt{mod} $K$, where $K$ is the number of clusters and
          \texttt{mod} returns the remainder when its first argument
          is divided by its second.
	\item[Sorted] Randomly determine an ordering of ground-truth
          clusters. All points of cluster $C_i$ arrive before any
          point of cluster $C_{i+1}$ arrives.
\end{description}

\begin{table}
\centering
  \begin{tabular}{c c c }
\hline
\bf Method  & Round. & Sort. \\
\hline
\alg      & 0.503 & 0.457 \\
\perch        & 0.446  &  0.351 \\
MB-HAC (5K) & 0.299  &  0.464 \\
MB-HAC (2K) & 0.171  &  0.451 \\
\hline
\end{tabular}
\caption{DP for adversarial arrival orders (ALOI).}
\label{tab:robust-dp}
\end{table}
\noindent As in previous work, we perform a robustness experiments
with the ALOI dataset. Table \ref{tab:robust-dp} shows that \alg
achieves higher dendrogram purity than both \perch and mini-batch HAC
(with 2 different batch sizes) on data ordered using the Round Robin
ordering scheme. Under this arrival order, MB-HAC performs poorly
showing its lack of robustness. When the data is in Sorted
order--which makes for easier clustering for MB-HAC--\alg outperforms
\perch and is competitive with MB-HAC.
\section{Related Work}
\label{sec:related}
The family of online and incremental clustering methods is diverse,
however all algorithms in this family optimize for specific
\hofs. \perch, from which the \rotate procedure is inspired, performs
rearrangments to satisfy a condition similar to
complete-linkage~\cite{kobren2017hierarchical}. BIRCH is another
top-down hierarchical clustering algorithm that attempts to minimize a
$k$-center style cost at each node in the
tree~\cite{zhang1996birch}. BIRCH also includes a non-greedy
reassignment step but has been shown to produce low quality trees in
practice. Liberty et al propose a flat clustering algorithm that
optimizes $k$-means cost. Since their algorithm runs in the online
setting, after a data point arrives and is assigned to a cluster, it
may never be reassigned~\cite{liberty2016algorithm}. While not
incremental, some work focuses on designing highly scalable algorithms
for specific linkage functions. Particular attention is paid to
single-linkage because of its connection to the minimum spanning tree
problem. For example, recent work develops massively parallel
algorithms for single-linkage~\cite{bateni2017affinity}.

When clustering with \hofs, probabilistic approaches can provide an
alternative to \hac. For example, split-merge Markov Chain Monte Carlo
(MCMC) methods perform clustering by randomly splitting and merging
clusters according to a proposal function~\cite{jain2004split}. An
algorithm similar to split-merge MCMC has even been used for author
coreference~\cite{wick2012discriminative}. This algorithm employs a
custom \hof on structured records and works by maintaining a
forest--each tree corresponding to a cluster--and randomly proposing
mergers and splits of various branches. Unlike \alg, this algorithm
relies on sampling to escape local minima. As the number of items
grows, the likelihood of sampling a merge or split that will be
accepted decreases rapidly.

Our work is partially inspired by complex \hofs that are used for
clustering. One example is Bayesian hierarchical clustering (BHC)--a
recursive, probabilistic, hierarchical model for
data~\cite{heller2005bayesian}. Fitting BHC models is performed by
running \hac with BHC as the \hof.  Because \hac is inefficient,
randomized approaches for fitting BHC have also been proposed, but
each of these methods still runs \hac as a subroutine on small,
randomly selected subsets of
data~\cite{heller2005randomized}. \hac-style algorithms are also used
to do probabilistic, hierarchical community detection and alongside
learned models for entity resolution~\cite{blundell2013bayesian,
  lee2012joint}.

Model-based separation is related to recently proposed definitions of
\emph{perfect hierarchical clustering structure}
\cite{cohen2018hierarchical,wang2018improved}, in which pairwise
similarities between data points lead to a tree that can be discovered
by \hac that has minimal cost. The costs used in these works are
variants of Dasgupta's cost~\cite{dasgupta2015cost}. Perfect
hierarchical clustering structures are a special case of model-based
separation, in which single-, average-, or complete-linkage is
used. Model-based separation is strictly more general, allowing for
linkage functions that compute the similarity of two point sets
arbitrarily, rather than as a function of pairwise data point
similarities.

\section{Conclusion}
\label{sec:conclusion}
This paper introduces \alg, an incremental algorithm for hierarchical
clustering under any \hof. The algorithm relies on two subroutines,
\rotate and \graft, that help it to discover complex cluster structure
regardless of data arrival order. We introduce model-based separation
for clustering with \hofs and prove that \alg always returns a tree
with perfect dendrogram purity when running in the separated setting.
We describe an efficient implementation of \alg and present an
empirical evaluation demonstrating that \alg is more accurate than
other baseline approaches and more scalable than \hac. We believe that
\alg is an asset for large clustering problems in which the
\records engage in complicated relationships and clusters are best
modeled by learned \hof.

Source code for \alg is available at: \url{https://github.com/iesl/grinch}.

\newpage

\bibliographystyle{ACM-Reference-Format}
\bibliography{ms}

\newpage

\begin{appendix}
\label{sec:appendix}

\section{Proof of Theorem \ref{thm:hsep}}
Define the following properties:

\begin{definition}[Strong Connectivity]
\label{def:strong}
Let $G=(\Xcal,E)$ be a graph and let $\Tcal[v]$ be a tree rooted at a
node $v$ with leaves, $\lvs{v}=\Xcal' \subseteq \Xcal$.  $v$ is
\textbf{connected} if $\Xcal'$ is a connected subgraph of $G$. $v$ is
\textbf{strongly connected} if every descendant of $v$ is connected. $v$
is a \textbf{maximal} strongly connected node if $v$ is strongly
connected and $\parent{v}$ is not strongly connected. Finally, the
tree $\Tcal$ satisfies \textbf{strong connectivity} if all connected
nodes in $\Tcal$ are strongly connected.
\end{definition}

\begin{definition}[Completeness]
\label{def:strong}
Let $G=(\Xcal,E)$ be a graph and let $\Tcal[v]$ be a tree rooted at a
node $v$ with leaves, $\lvs{v}=\Xcal' \subseteq \Xcal$.  Then $v$ is
\textbf{complete} if $\Xcal'$ is a connected component in $G$. The tree
$\Tcal$ satisfies \textbf{completeness} if the set of connected components
of $G$ are (the leaves of) a tree consistent partition of $\Tcal$.
\end{definition}

According to Theorem \ref{thm:hsep}, \alg always constructs a tree
that satisfies completeness. To prove the theorem, we will show that
after the addition of each new \record, the resulting tree satisfies
strong connectivity and completeness. We analyze various subroutines
of \alg and demonstrate how they preserve strong connectivity,
completeness or both.  In the proceeding lemmas and proofs, let
$G=(\Xcal, E)$ be a graph and let $f$ be a model that separates $G$.

\begin{lemma}[Rotation Lemma]
\label{thm:rotate}
Let $\Tcal$ be a tree with $\lvs{\Tcal} = \Xcal$, and let $x$ be a new
\record to be added to $\Tcal$. Then all nodes that were strongly
connected before the addition of $x$ are strongly connected after the
addition of $x$, i.e., rotations preserve strong connectivity.
\end{lemma}

Note: while rotations preserve strong connectivity, they do not
guarantee completeness. Therefore rotations are insufficient for
proving Theorem \ref{thm:hsep}.

\begin{proof}
Let $v$ be a maximal strongly connected node in $\Tcal$ and assume
that $x$ is added as a leaf of $v$ (rotations have not yet been
applied). Consider two cases:
\begin{enumerate*}[label=(\arabic*)]
\item there exists an edge between $x$ and some leaf in $\lvs{v}$,
  and
\item there does not exist an edge between $x$ and any leaf in
  $\lvs{v}$.
\end{enumerate*}

\begin{figure*}[t!]
  \captionsetup[subfigure]{justification=centering}
  \centering
  \begin{subfigure}[h]{0.5\textwidth}
  \centerline{\includegraphics[width=\textwidth]{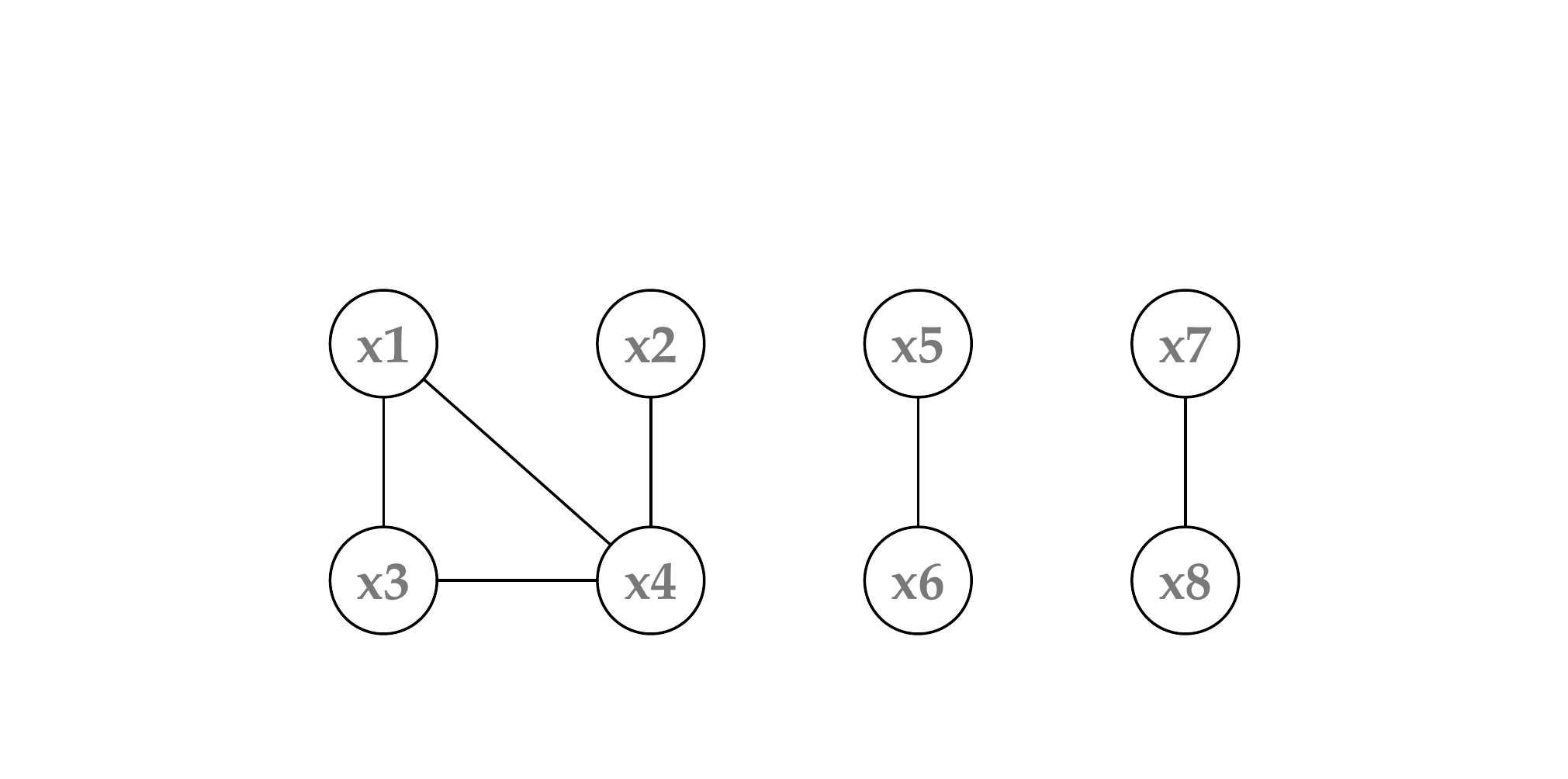}}
  \caption{A graph $G = (\Xcal, E)$.}
    \label{fig:grinch-graph}
\end{subfigure}\\
\begin{subfigure}[h]{0.49\textwidth}
  \centerline{\includegraphics[width=\textwidth]{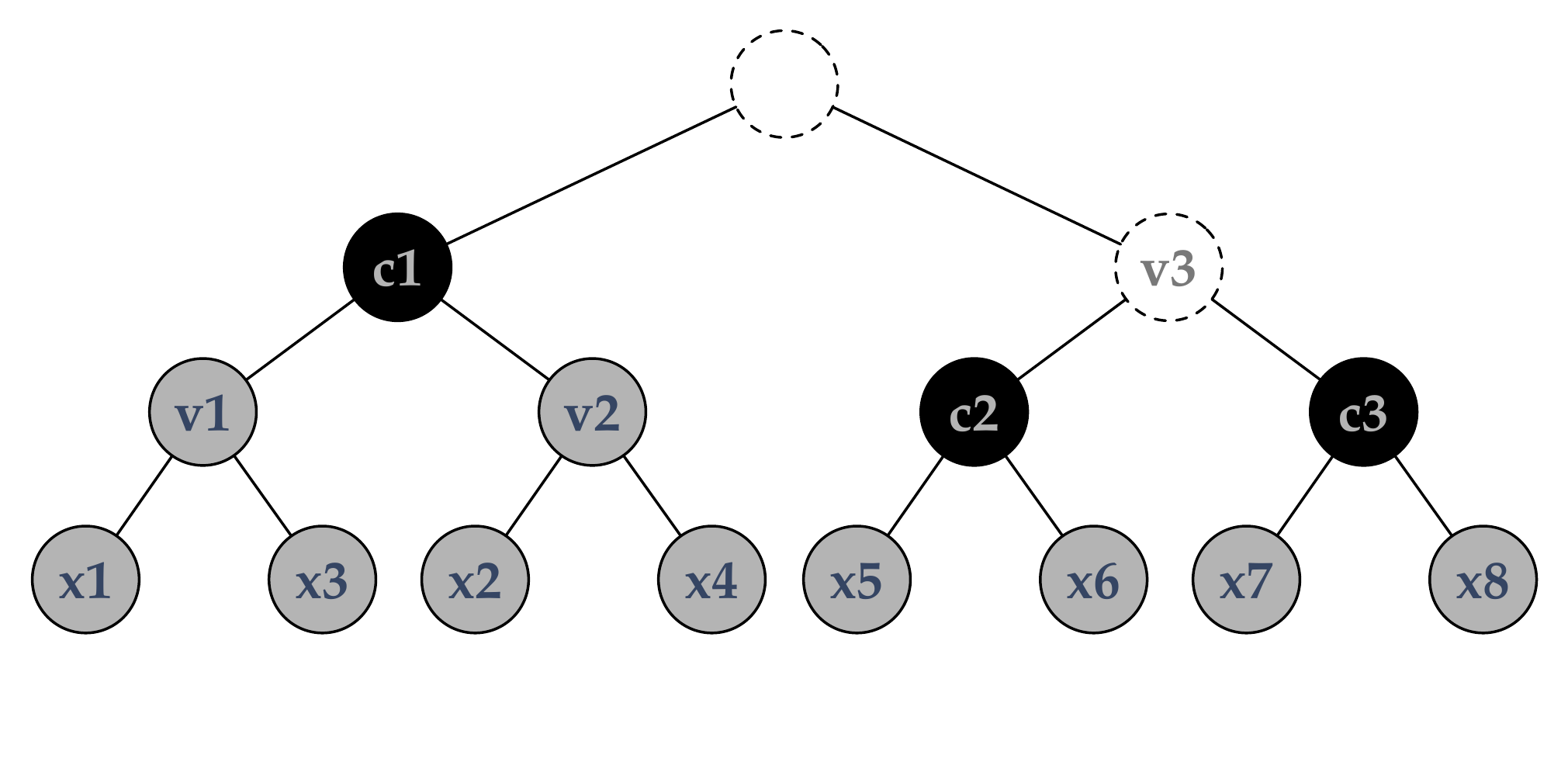}}
  \caption{Strongly connected \& complete.}
    \label{fig:strong-complete}
\end{subfigure}
\begin{subfigure}[h]{0.49\textwidth}
  \centerline{\includegraphics[width=\textwidth]{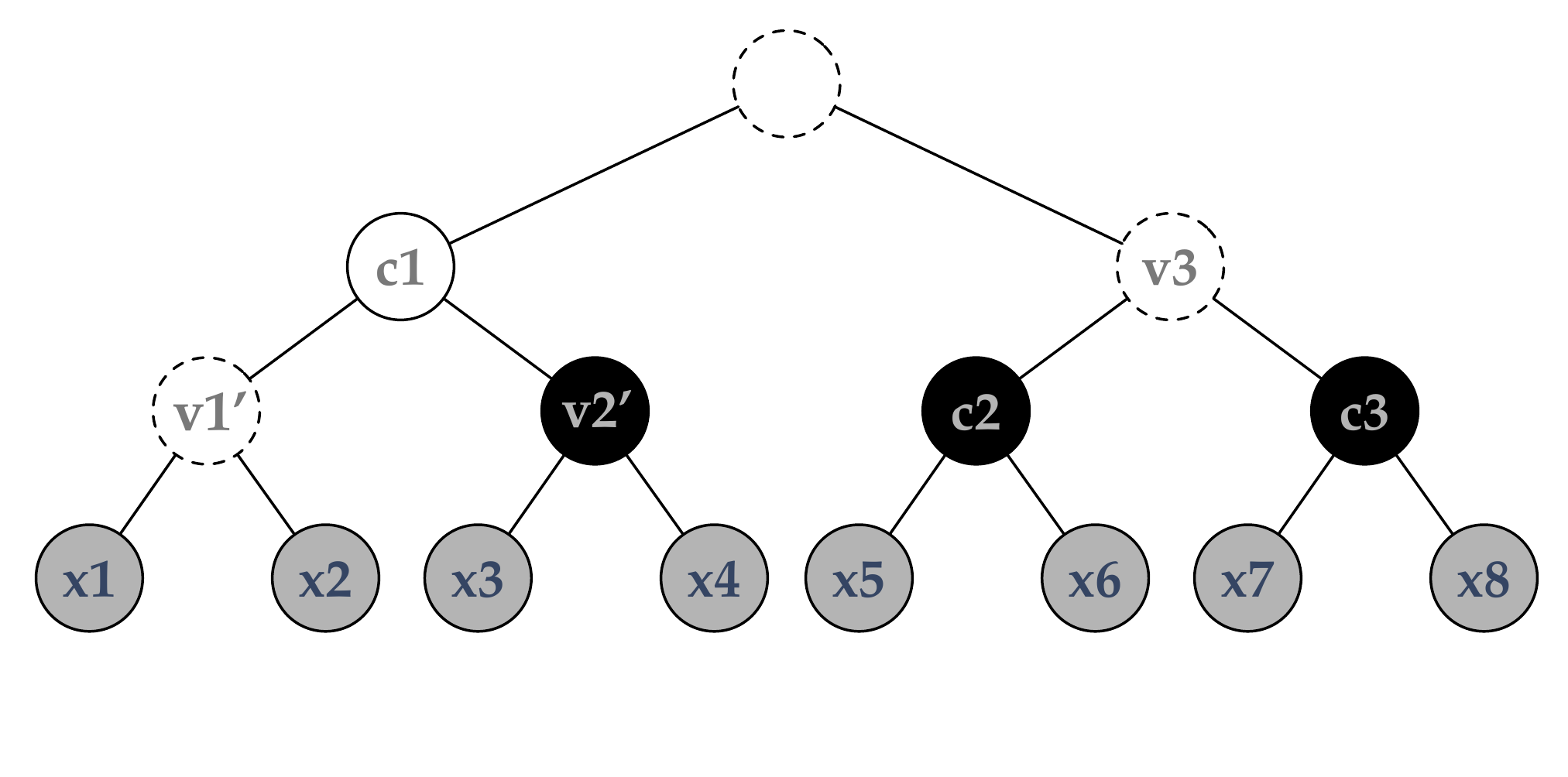}}
  \caption{Complete only.}
    \label{fig:complete-only}
\end{subfigure}
\caption{A graph $G$ with 3 connected components (Figure
  \ref{fig:grinch-graph}).  In Figure \ref{fig:strong-complete} and
  Figure \ref{fig:complete-only}, black-filled nodes are maximal,
  gray-filled nodes are strongly connected, nodes with no fill and
  solid borders are connected (but not strongly), and nodes with dashed
  borders are disconnected. The tree in Figure
  \ref{fig:strong-complete} satisfies strong connectivity and
  completeness. The tree in Figure \ref{fig:complete-only} does not
  satisfy strong connectivity because $v_1$ is disconnected.}
\label{fig:graph-strong-complete}
\end{figure*}

\paragraph{Case 1:}
Let $L \subseteq \lvs{v}$ be the set of $v$'s descendant leaves to
which $x$ is connected. Then $x$ is initially added as a sibling of
its nearest neighbor leaf, $x'$, and $x' \in L$ because $f$ separates
$G$. $\parent{x}$ is strongly connected because there exists an edge
between $x$ and $x'$.

The addition of $x$ does not disconnect $v$ or any strongly connected
descendant of $v$. To see why, consider the siblings of the ancestors
of $x'$ before the addition of $x$. Any such sibling that was
connected to $x'$, is, after the addition of $x$, also connected to
$\parent{x}$ and thus remains strongly connected. Nodes that are not
ancestors of $x$ cannot be disconnected and thus, before rotations,
strong connectivity is preserved.

Now consider subsequent rotations. By the logic above, $x$ and its
sibling, $x' = \sib{x}$, are connected. If a rotation succeeds then $x$ and
$\aunt{x}$ are swapped. So long as $\aunt{x}$ and $\sib{x}$ form a
connected subgraph in $G$, i.e., $\phi(\sib{x}, \aunt{x}) = \phi(x,
\sib{x}) = 1$, then the rotation preserves strong connectivity.

The only way for a rotation to disrupt strong connectivity is if $x$
and $\aunt{x}$ are swapped, and $\sib{x}$ and $\aunt{x}$ do not form a
connected subgraph in $G$, i.e., $\phi(x, \sib{x}) > \phi(\sib{x},
\aunt{x})$. But, because $f$ separates $G$, $\phi(x, \sib{x}) >
\phi(\sib{x}, \aunt{x}) \Longrightarrow f(x, \sib{x}) > f(\sib{x},
\aunt{x})$ and so, in this case, a rotation will not be performed and
the procedure terminates.

\paragraph{Case 2:}
If there does not exist an edge between $x$ and any leaf in $\lvs{v}$,
then after $x$ is made a sibling of some leaf $x'' \in \lvs{v}$, $v$
is no longer strongly connected and so strong connectivity has not
been preserved. Since $v$ was strongly connected before the addition
of $x$, there exists an edge between $\lvs{\sib{x}}$ and
$\lvs{\aunt{x}}$.  Since $f$ separates $G$, $f(x, \sib{x}) <
f(\sib{x}, \aunt{x})$, which triggers the \rotate
subroutine. Rotations proceed with respect to $x$ at least until $x$
is no longer a descendant of $v$, and thus, $v$ remains strongly
connected. Strongly connected nodes that are not descendants of $v$
are unaffected by the rotations and so strong connectivity is
preserved.
\end{proof}

\begin{lemma}[Grafting Lemma 1]
\label{thm:grafting1}
Let $\Tcal$ satisfy strong connectivity and completeness. Let $v$ be a
node in $\Tcal$ such that $v$ is either a maximal strongly connected
node or not strongly connected. Then a \graft operation initiated from
$v$ preserves strong connectivity and completeness.
\end{lemma}

\begin{proof}
Let $\Tcal$ be strongly connected and complete. Since $\lvs{v}$ is not
a strict subset of any connected component in $G$, there does not
exist a non-empty subset $s$ in $\lvs{\Tcal} \backslash \lvs{v}$ such
that $s \cup \lvs{v}$ is a connected subgraph in $G$. For any node
$v'$ that is strongly connected but not maximal, there must be an edge
connecting $\lvs{v'}$ and $\lvs{\sib{v'}}$ and $\sib{v'}$ must be
strongly connected, so $f(v', \sib{v'}) > f(v, v')$.  Therefore, an
attempt to make any such $v'$ the sibling of $v$ fails.

If $v''$ is a maximal strongly connected node, an attempt to make
$v''$ the sibling of $v$ may succeed but this does not disconnect any
strongly connected subtrees in $\Tcal$.  The same is true if $v''$ is
not strongly connected.
\end{proof}

\begin{figure*}[t!]
\captionsetup[subfigure]{justification=centering}
\begin{subfigure}[h]{0.49\textwidth}
  \centerline{\includegraphics[width=\textwidth]{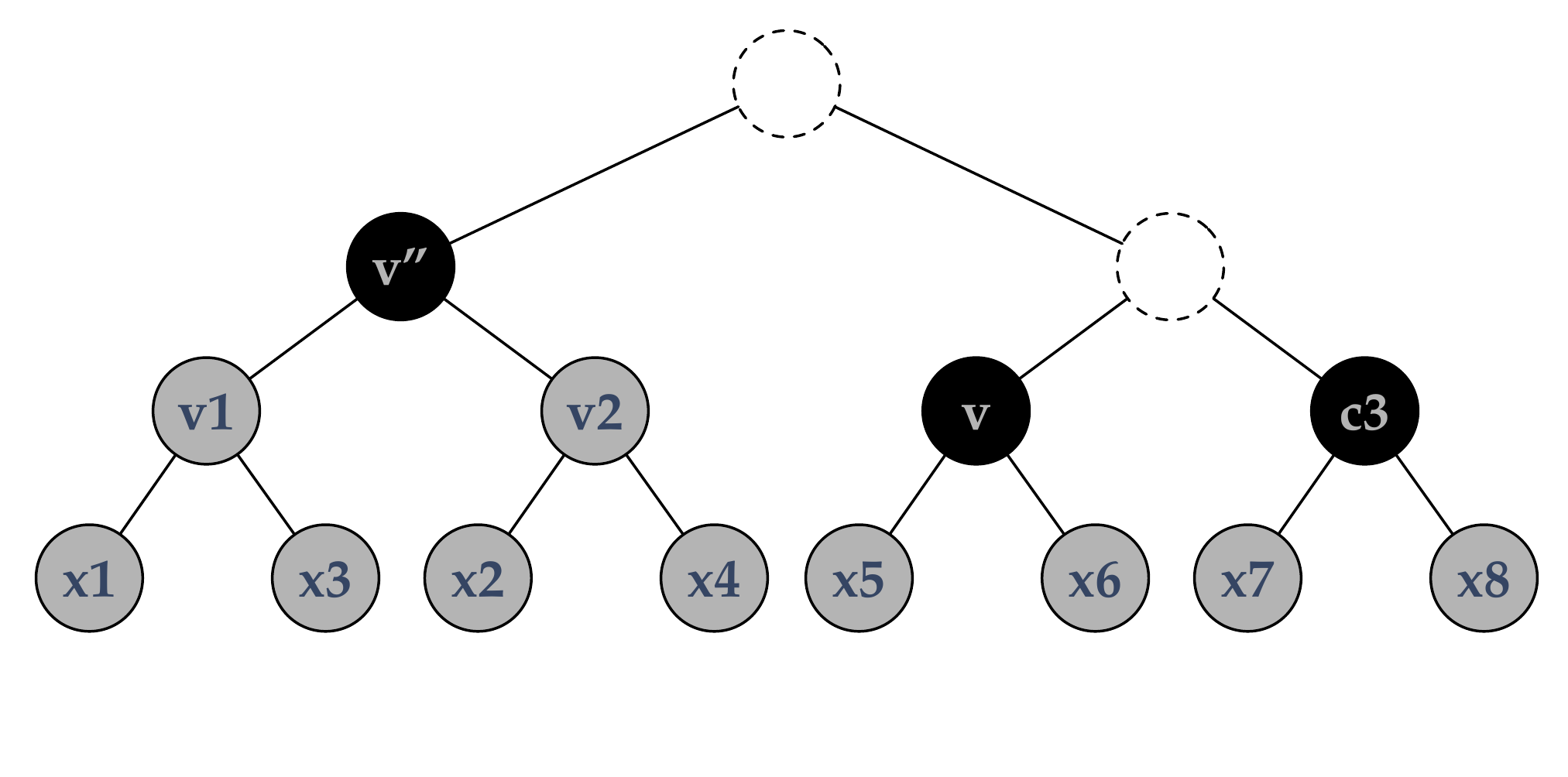}}
  \caption{Grafting Lemma 1 Visual Aid.}
    \label{fig:graft-lem1}
\end{subfigure}
\begin{subfigure}[h]{0.49\textwidth}
  \centerline{\includegraphics[width=\textwidth]{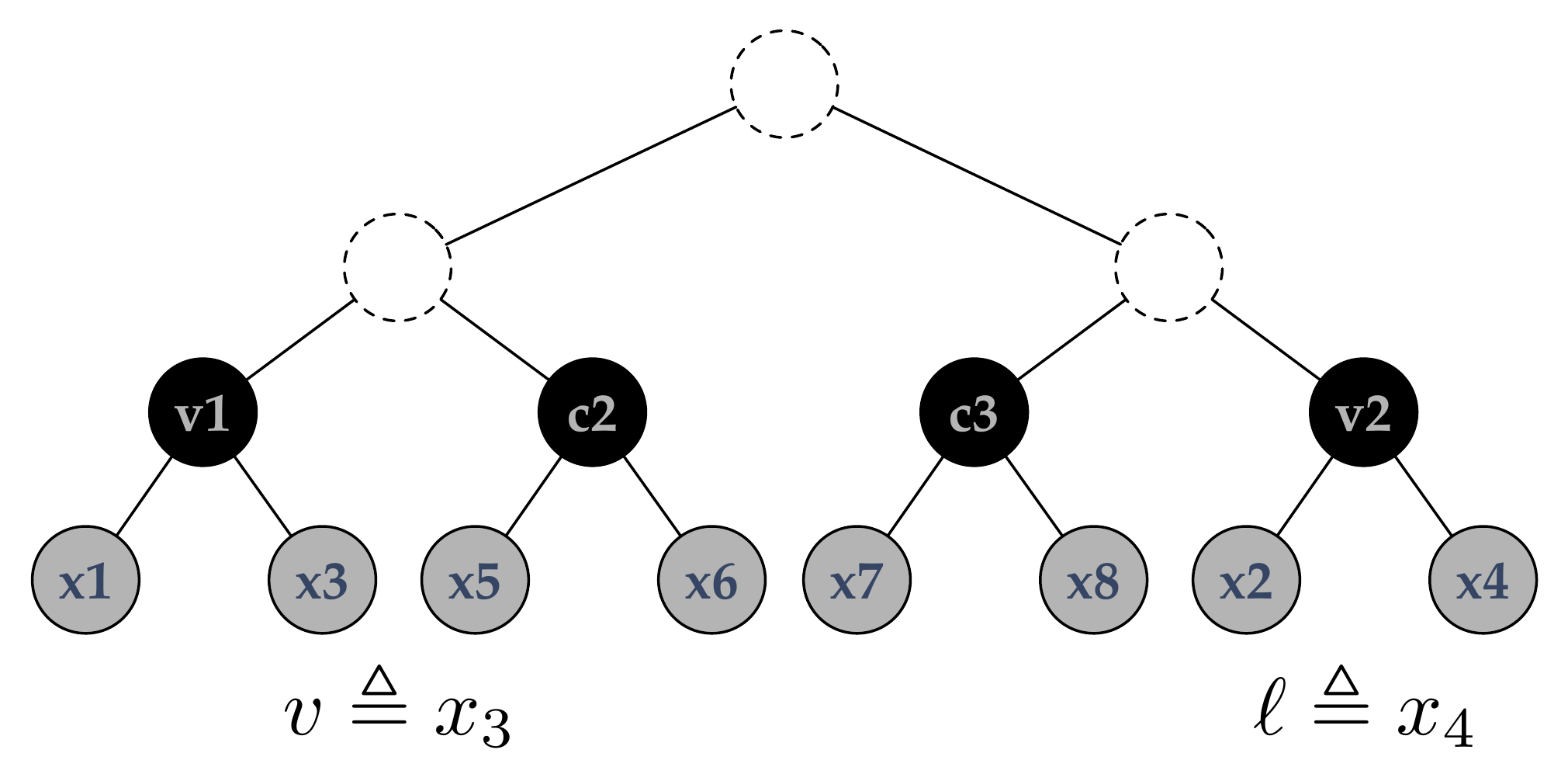}}
  \caption{Grafting Lemma 2 Visual Aid.}
  \label{fig:graft-lem2}
\end{subfigure}
\caption{We reuse the graph in Figure \ref{fig:grinch-graph}. The tree
  in Figure \ref{fig:graft-lem1} is strongly connected and
  complete. Consider the node $v$. A graft initiated from $v$ may make
  $v''$ a sibling of $v$ because $\lvs{v}$ is not a (strict) subset of
  a connected component and $v''$ is maximal. After such a graft,
  notice that the tree would still satisfy strong connectivity and
  completeness.  The tree in Figure \ref{fig:graft-lem2} is strongly
  connected but not complete. Consider $x_3$ which plays the role of
  $v$ in the proof of Grafting Lemma 2. When a constrained nearest
  neighbor search is executed from its parent, $v_1$, the leaf
  $x_4$--which plays the role of $\ell$--is returned. If $v_1$ and
  $v_2$ are made siblings, their parent is strongly
  connected.}
\label{fig:graft-lem-12}
\end{figure*}

\begin{lemma}[Grafting Lemma 2]
\label{thm:grafting1}
Let $\Tcal$ be a tree such that $\lvs{\Tcal} = \Xcal$ and let $\Tcal$
satisfy strong connectivity. Let $v$ be strongly connected and let
$\lvs{v}$ be a strict subset of the vertices in some connected
component, $C$, in $G$.  Then, a \graft initiated from $v$ returns a
node $v'$ such that $v'$ is strongly connected and $\lvs{v} \subset
\lvs{v'}$.
\end{lemma}

\begin{proof}
Since $\lvs{v}$ are a strict subset of the vertices in the connected
component, $C$, there exists a non-empty subset $s$ in $\lvs{\Tcal}
\backslash \lvs{v}$ such that $s \cup \lvs{v}$ constitute the vertices
in $C$. Let $\ell$ maximize $f(v,\ell)$ over all $\lvs{\Tcal}
\setminus \lvs{v}$. By the fact that $\lvs{v}$ is a strict subset of a
connected component, there must exist an edge between $\lvs{v}$ and
$\ell$. Note that $\ell$ is the leaf found when the constrained
nearest neighbor search from $v$ is initiated in the first line of
\graft (Algorithm \ref{alg:graft}).

If $f(v, \ell) < f(\ell, \sib{\ell})$, then there must exist an edge
between $\ell$ and a node in $\lvs{\sib{\ell}}$ and
so $\parent{\ell}$ is strongly connected.
If $f(v, \ell) < f(v, \sib{v})$, then there must
exist an edge between a node in $\lvs{v}$ and a node in
 $\lvs{\sib{v}}$ and so $\parent{v}$ is strongly connected.
  In both of these cases, we do not
merge $v$ with $\ell$, but instead attempt another merge between two
strongly connected nodes, either: $\parent{v}$ with $\ell$, $v$ with
$\parent{\ell}$, or $\parent{v}$ with $\parent{\ell}$. As before, the
two nodes we are attempting to merge also have an edge between them.

Let $v_1$ and $v_2$ be two nodes involved in a merge and let $v_1 \in
\ancs{v}$ and $v_2 \in \ancs{\ell}$.  If at some point
\begin{align*}
  f(v_1, v_2) >\max[f(v_1, \sib{v_1}), f(v_2, \sib{v_2})]
\end{align*}
then $v_2$ is made a
sibling of $v_1$ and the new parent of $v_1$ is returned. Since $v_1$
and $v_2$ are strongly connected and there exists an edge between
\lvs{v_1} and \lvs{v_2}, $\parent{v_1}$, which is created by the merge, is
strongly connected, and the lemma holds.

If a merge is never performed, the recursion stops when $v_1 = v_2 =
\lca{v}{\ell}$. In this case, the \texttt{lca}, which we return, is
already strongly connected and, by definition, its leaves are a
superset of $\lvs{v}$.
\end{proof}

\begin{lemma}[Restructuring Lemma]
Let $v \in \Tcal$ be strongly connected. Let $a \in \ancs{v}$ be the
deepest connected ancestor of $v$ such that: $a$ is not strongly
connected, and all siblings of the nodes on the path from $v$ to $a$
are strongly connected. Then \rst on inputs $v$ and $a$ restructures
$\Tcal[a]$ so that $a$ satisfies strong connectivity.
\end{lemma}

\begin{figure*}[h]
\captionsetup[subfigure]{justification=centering}
\centering
\begin{subfigure}[h]{0.49\textwidth}
  \centerline{\includegraphics[width=\textwidth]{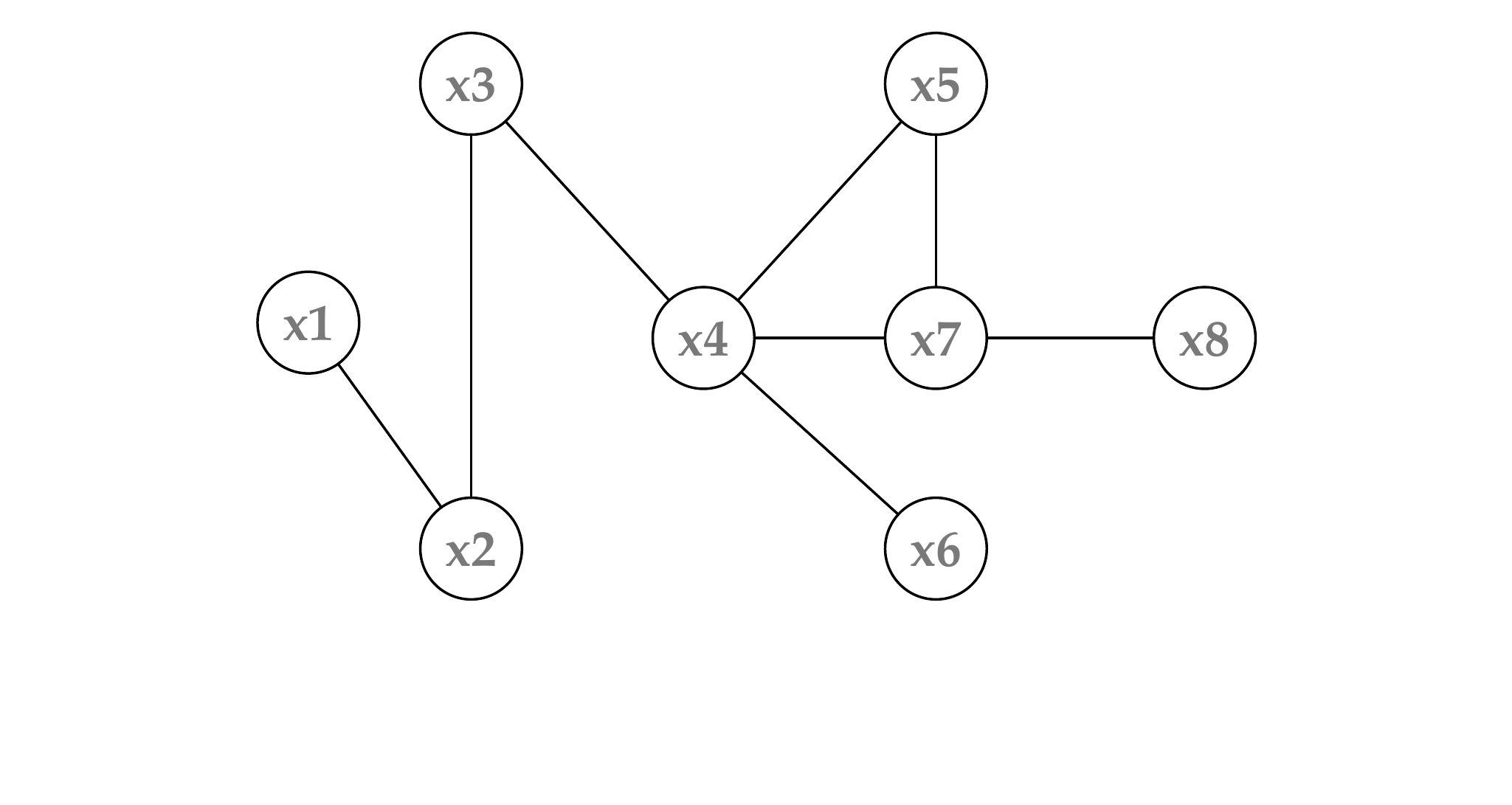}}
  \caption{A graph $G = (\Xcal, E)$.}
  \label{fig:restruct-lem-graph}
\end{subfigure}
\begin{subfigure}[h]{0.49\textwidth}
  \centerline{\includegraphics[width=\textwidth]{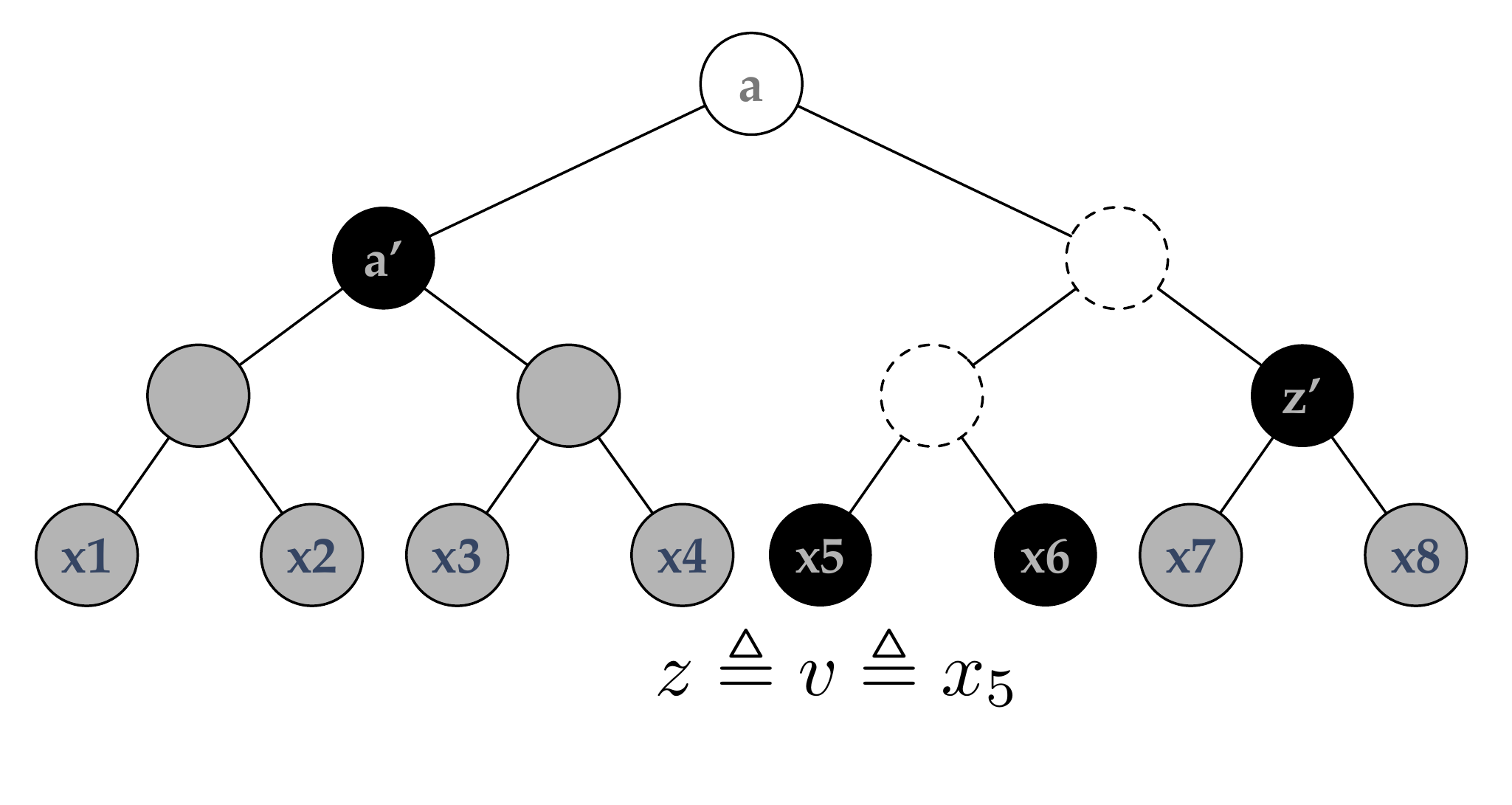}}
  \caption{$a$ is connected, but not strongly.}
  \label{fig:restruct-lem-pt1}
\end{subfigure}\\
\begin{subfigure}[h]{0.49\textwidth}
  \centerline{\includegraphics[width=\textwidth]{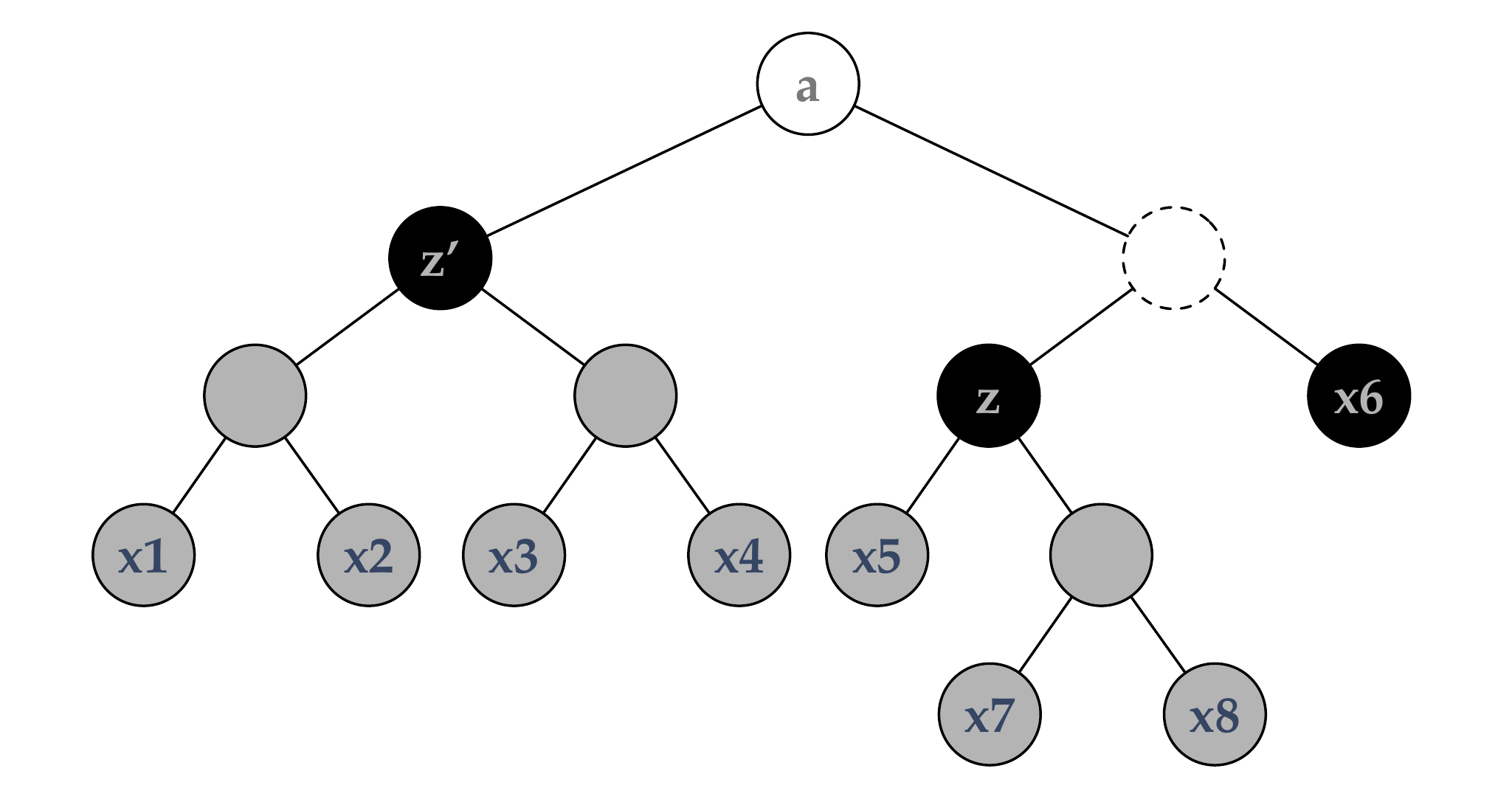}}
  \caption{1 swap applied.}
  \label{fig:restruct-lem-pt2}
\end{subfigure}
\begin{subfigure}[h]{0.49\textwidth}
  \centerline{\includegraphics[width=\textwidth]{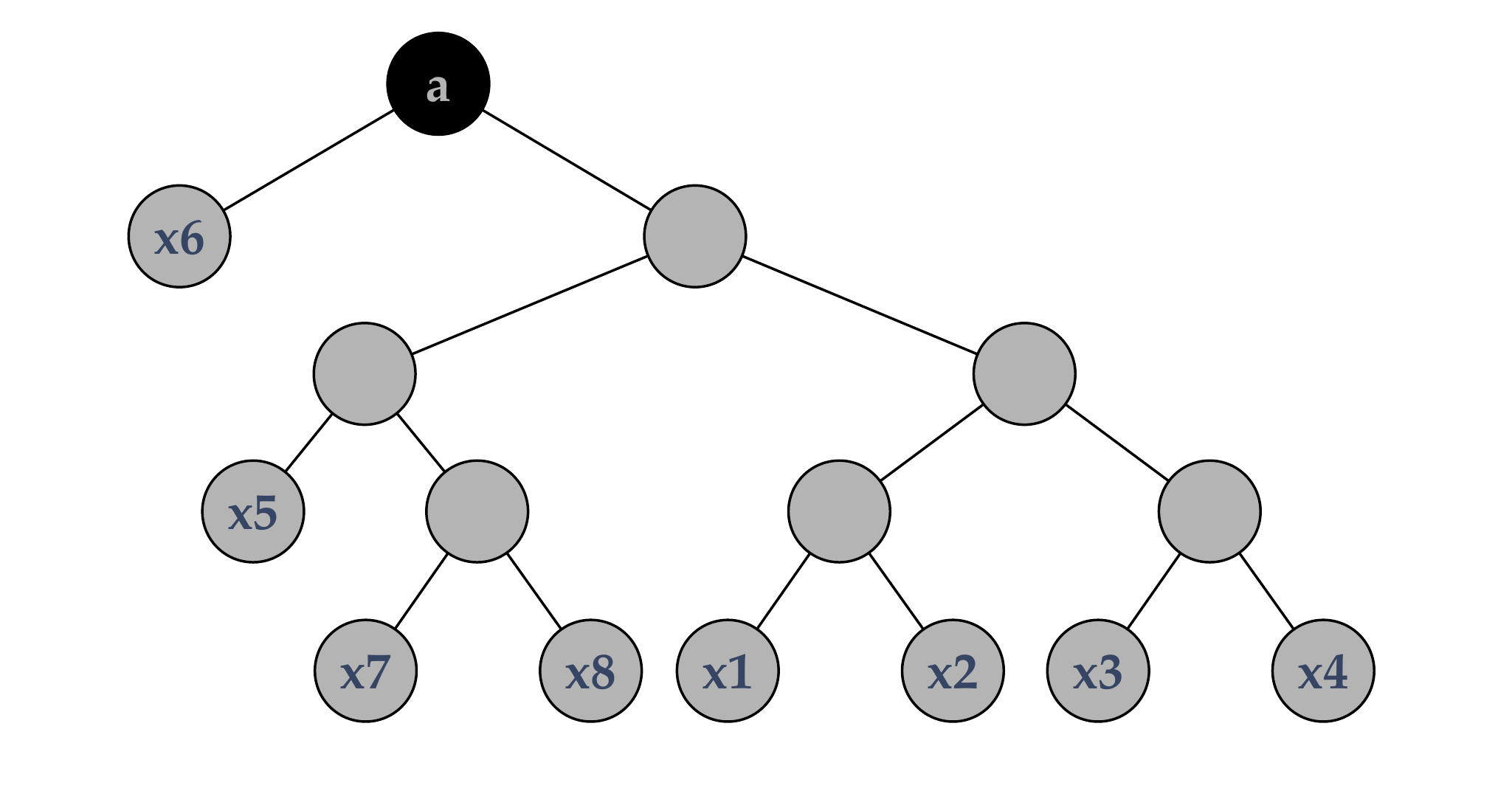}}
  \caption{$a$ is strongly connected.}
  \label{fig:restruct-lem-pt3}
\end{subfigure}
\caption{The \rst method. As before, black-filled nodes are maximal,
  gray-filled nodes are strongly connected, nodes with no fill and
  solid borders are connected (but not strongly) and nodes with dashed
  borders are disconnected. Also, note that the labels $z, z',\ell$
  and $a'$ do not apply to the same nodes throughout all figures so to
  match their usage in proof. In Figure \ref{fig:restruct-lem-pt1},
  $z \triangleq v \triangleq x_5$.  $\sib{z}$ and $z'$ are swapped to
  produce the tree in Figure \ref{fig:restruct-lem-pt2}. Finally,
  $\sib{z}$ and $z'$ from Figure \ref{fig:restruct-lem-pt2} are
  swapped to produce the tree in Figure \ref{fig:restruct-lem-pt3},
  which is strongly connected.}
  \label{fig:restruct-lem}
\end{figure*}

\begin{proof}
Let $z$ be the deepest ancestor of $v$ that is strongly connected with
parent $\parent{z}$ that is disconnected. Since $\parent{z}$ is
disconnected (but by assumption both $z$ and $\sib{z}$ are connected),
there are no edges between \lvs{z} and \lvs{\sib{z}}.

Let $a'$ be a child of $a$ and without loss of generality, $a' \not\in
\ancs{z}$.  Since $a$ is the deepest connected ancestor of $z$, there
must exist an edge between \lvs{z} and \lvs{a'}.

When computing the argmax of $f(z, \cdot)$ in the \rst method, a node,
$z'$, that is connected to $z$ will be returned and then swapped with
$\sib{z}$. The new parent of $z$ is strongly connected because $z$ and
$z'$ are both strongly connected and there exists an edge between
\lvs{z} and \lvs{z'}. Any subsequent swap attempted from a
disconnected node with a connected ancestor succeeds and produces a
new parent that is strongly connected.

Since $a$ is connected and a swap among the descendants of
$a$ do not change $\lvs{a}$, swapping preserves the connectedness of
$a$. Therefore, swaps proceed until the node $a$ is reached at which
point $a$ must be strongly connected.

Note that a swap attempt between a strongly connected node and a node
to which it is not connected fails, because $f$ separates $G$.  A swap
attempt between a connected node and a node to which it is connected
succeeds and produces a new parent that is strongly connected.
\end{proof}

We now prove Theorem \ref{thm:hsep}.

\begin{proof}
We show by induction that if \alg is used to build a tree, $\Tcal$,
over vertices, $\Xcal$, then the connected components of $G$ are a
tree consistent partition in $\Tcal$. Furthermore, $\Tcal$ satisfies
strong connectivity.

Clearly, the theorem holds for the base case: a tree with a single
node.

Let $\Xcal = \lvs{\Tcal}$. Assume the inductive hypothesis: that
$\Tcal$ satisfies completeness and strong connectivity. Now vertex $x$
arrives.

If there does not exist an edge between $x$ and any other vertex in
$\Xcal$, then after rotations, $\Tcal'$ satisfies completeness. Since
$\forall a \in \ancs{x}$, \lvs{a} is a not a strict subset of any
connected component in $G$, by Grafting Lemma 1, subsequent \graft
attempts from the ancestors of $x$ preserve strong connectivity and
completeness and so the theorem holds.

Assume that $x$ is connected to some set of leaves $s \subseteq
\lvs{\Tcal}$.  Since $\Tcal$ satisfies strong connectivity, by the
Rotation Lemma, after $x$ is added and rotations terminate, $\Tcal'$
satisfies strong connectivity. Note that $\Tcal'$ may not satisfy
completeness if, before the arrival of $x$, the leaves in $s$ formed
at least 2 distinct connected subgraphs in $G$.

After rotations, a series of \graft attempts are performed. Consider
the first \graft initiated at $\parent{x}$.
By Grafting Lemma 2, the
attempt returns a strongly connected ancestor of $x$ whose leaves are
a strict superset of $\lvs{x}$.  If a \texttt{merge} is performed that
moves a node $v$ and makes it a sibling of $v'$, then strong
connectivity may be violated. However, notice that the only nodes that
can be disconnected by such a merge are the node that, prior to the
merge, were ancestors of $v$ and also descendants of $a = \lca{v}{v'}$.

After the merge, $a$ is restructured, and by the Restructuring Lemma,
the resulting tree satisfies strong connectivity. Subsequent calls to
\graft proceed from $a$. Notice that each invocation of \graft returns
a new strongly connected node with a strictly larger number of
descendant leaves, until the resulting tree satisfies completeness.
Therefore, successive grafting followed by restructuring eventually
returns a node whose leaves are a connected component of
$G$. Ultimately, after rotations and grafting, $\Tcal'$ must satisfy
completeness and strong connectivity.
\end{proof}
\end{appendix}

\end{document}